\documentclass[sigconf, nonacm]{acmart}

\setcopyright{none}
\usepackage{amsmath}
\usepackage{colortbl}
\usepackage{multirow}
\usepackage{tabularx}
\newcolumntype{Y}{>{\raggedright\arraybackslash}X}
\usepackage{enumitem}
\usepackage{xspace}
\usepackage{tikz}
\usepackage{subcaption}
\usetikzlibrary{arrows.meta,calc,positioning,shapes.geometric,patterns}
\usepackage{pgfplots}
\pgfplotsset{compat=1.18}
\usepgfplotslibrary{groupplots}
\usepackage[most]{tcolorbox}
\usetikzlibrary{shadows}
\definecolor{pvAtom}{HTML}{E3ECF6}\definecolor{pvAtomInk}{HTML}{2F5A88}\definecolor{pvExist}{HTML}{FBEEDC}\definecolor{pvExistInk}{HTML}{9A5B12}
\newtcolorbox{pvariant}[2]{colback=#1, colframe=#1, boxrule=0pt, arc=1.5pt, left=3pt, right=3pt, top=0pt, bottom=0pt,
  before skip=2pt, after skip=2pt, title={#2}, fonttitle=\scriptsize\bfseries, coltitle=#1Ink, colbacktitle=#1,
  toptitle=1pt, bottomtitle=0pt, titlerule=0pt}
\newtcolorbox{promptbox}[1]{breakable, width=\linewidth, colback=white, colframe=black, boxrule=0.8pt, arc=2pt,
  left=5pt, right=5pt, top=3pt, bottom=3pt, before skip=6pt, after skip=6pt, toptitle=1pt, bottomtitle=1pt, fonttitle=\bfseries, title={#1}}
\usepackage{pifont}

\newcommand{\x}{\ensuremath{\mathbf{x}}\xspace}

\newcommand{\tran}{\ensuremath{\mathit{\tran}}}

\newcommand{\Agg}{\textit{\Agg}}

\newcommand{\Omc}{\ensuremath{\mathcal{O}}\xspace}

\newcommand{\NC}{\ensuremath{\mathsf{N_C}}\xspace}
\newcommand{\NR}{\ensuremath{\mathsf{N_R}}\xspace}

\usepackage{algorithm}
\usepackage{algorithmic}
\newcommand{\NA}{\makebox[2.6em][c]{---}}
\usepackage{newfloat}
\usepackage{amsthm}
\usepackage{listings}
\DeclareCaptionStyle{ruled}{labelfont=normalfont,labelsep=colon,strut=off}
\floatstyle{ruled}
\newfloat{listing}{tb}{lst}{}
\floatname{listing}{Listing}

\usepackage{booktabs}
\AtBeginEnvironment{table}{\renewcommand{\arraystretch}{0.95}}
\AtBeginEnvironment{table*}{\renewcommand{\arraystretch}{0.95}}
\usepackage[table]{xcolor}
\usepackage{seqsplit}

\usepackage{tikz}
\usetikzlibrary{shapes.geometric, arrows, decorations.markings, shadings}
\usetikzlibrary{arrows.meta,positioning,decorations.pathmorphing}

\tikzstyle{startstop} = [rectangle, rounded corners, minimum width=3cm, minimum height=1cm,text centered, draw=black, fill=red!30]
\tikzstyle{process} = [rectangle, minimum width=3cm, minimum height=1cm, text centered, draw=black, fill=blue!30]
\tikzstyle{arrow} = [thick,->,>=stealth]

\usepackage{algorithm}
\usepackage{algorithmic}
\usepackage{amsfonts} 
\usepackage{amsmath}

\newtheorem{rem}{Remark}

\newtheorem{example}{Example}

\usepackage{bm}
\usepackage{xspace}
\usepackage{xcolor}
\usepackage{multirow}
\usepackage{caption}

\newtheorem{definition}{Definition}
\newcommand{\ours}{\texttt{NeurOWL}\xspace}
\newcommand{\hl}[1]{\cellcolor{gray!15}#1}

\usepackage{tikz}
\usetikzlibrary{arrows.meta,shapes.geometric,positioning,fit,backgrounds,calc}
\usepackage[dvipsnames,svgnames,x11names]{xcolor}

\definecolor{cPurpleF}{RGB}{238,237,254}\definecolor{cPurple}{RGB}{83,74,183}
\definecolor{cBlueF}{RGB}{230,241,251}  \definecolor{cBlue}{RGB}{24,95,165}
\definecolor{cAmberF}{RGB}{250,238,218} \definecolor{cAmber}{RGB}{186,117,23}
\definecolor{cCoralF}{RGB}{250,236,231} \definecolor{cCoral}{RGB}{153,60,29}
\definecolor{cTealF}{RGB}{225,245,238}  \definecolor{cTeal}{RGB}{15,110,86}
\definecolor{cPinkF}{RGB}{251,234,240}  \definecolor{cPink}{RGB}{153,53,86}
\definecolor{cRedF}{RGB}{252,235,235}   \definecolor{cRed}{RGB}{163,45,45}
\definecolor{cGrayF}{RGB}{241,239,232}  \definecolor{cGray}{RGB}{95,94,90}
\definecolor{cGrayM}{RGB}{180,178,169}
\definecolor{tabData}{HTML}{EAF0F7}
\definecolor{tabNeg}{HTML}{EBEBEB}

\tikzset{
  nd/.style={rectangle,rounded corners=4pt,inner sep=5pt,
             font=\scriptsize\sffamily,align=center,line width=0.4pt},
  Nq/.style ={nd,draw=cPurple!70,fill=cPurpleF,minimum width=58pt,minimum height=32pt},
  N1/.style ={nd,draw=cBlue!70,  fill=cBlueF,  minimum width=72pt,minimum height=38pt},
  N2a/.style={nd,draw=cAmber!70, fill=cAmberF, minimum width=110pt,minimum height=18pt},
  N2b/.style={nd,draw=cCoral!70, fill=cCoralF, minimum width=110pt,minimum height=18pt},
  N2c/.style={nd,draw=cTeal!70,  fill=cTealF,  minimum width=110pt,minimum height=15pt},
  N3/.style ={nd,draw=cPink!70,  fill=cPinkF,  minimum width=72pt,minimum height=44pt},
  Nok/.style={nd,draw=cTeal!70,  fill=cTealF,  minimum width=72pt,minimum height=28pt},
  Nno/.style={nd,draw=cRed!70,   fill=cRedF,   minimum width=72pt,minimum height=28pt},
  dec/.style={diamond,draw=cGray!70,fill=cGrayF,aspect=1.7,
              inner sep=1.5pt,font=\tiny\sffamily,line width=0.4pt},
  a/.style  ={-{Stealth[length=2.2pt,width=2pt]},line width=0.65pt,cGray},
  ya/.style ={a,cTeal},  ra/.style={a,cRed},
  da/.style ={a,cGrayM,dashed},
  lx/.style ={font=\tiny\sffamily,inner sep=1pt},
  yl/.style ={lx,text=cTeal},  rl/.style={lx,text=cRed},  gl/.style={lx,text=cGray},
  hd/.style ={font=\tiny\sffamily\bfseries,text=cGray},
}

\definecolor{c1col}{RGB}{30,100,200}
\definecolor{c2col}{RGB}{30,150,100}
\definecolor{c3col}{RGB}{200,120,20}
\definecolor{c4col}{RGB}{180,30,30}

\tikzset{
  given/.style  = {rectangle, rounded corners=3pt, draw=black, very thick,
                   fill=gray!12, inner sep=5pt, align=center, font=\small},
  pred1/.style  = {rectangle, rounded corners=3pt, draw=c1col, thick,
                   fill=c1col!10, inner sep=5pt, align=center, font=\small},
  pred2/.style  = {rectangle, rounded corners=3pt, draw=c2col, thick,
                   fill=c2col!10, inner sep=5pt, align=center, font=\small},
  pred3/.style  = {rectangle, rounded corners=3pt, draw=c3col, thick,
                   fill=c3col!10, inner sep=5pt, align=center, font=\small},
  pred4/.style  = {rectangle, rounded corners=3pt, draw=c4col, thick,
                   fill=c4col!10, inner sep=5pt, align=center, font=\small},
  solidarr/.style = {->, >=Stealth, thick, black},
  predarr/.style  = {->, >=Stealth, thick, red!80!black, dashed},
  openarr/.style  = {->, >=Stealth, thick, red!80!black, dashed,
                     decorate, decoration={snake, amplitude=1.5pt, segment length=7pt,
                                           pre length=4pt, post length=4pt}},
  elbl/.style = {font=\scriptsize\itshape, text=gray!80!black, pos=0.5, right, xshift=3pt},
}

\begin{document}
\title{\ours: An LLM-Based Neural-Symbolic Framework for Incomplete OWL Ontology Reasoning}

\author{Hui Yang\textsuperscript{1}, Jiaoyan Chen\textsuperscript{1}, Yiping Song\textsuperscript{1}, Renate Schmidt\textsuperscript{1}, Wen Zhang\textsuperscript{2}}
\affiliation{\institution{\textsuperscript{1}The University of Manchester, Manchester, United Kingdom\quad \textsuperscript{2}Zhejiang University, Hangzhou, China\\
    \{hui.yang-2, jiaoyan.chen, renate.schmidt\}@manchester.ac.uk, yiping.song@postgrad.manchester.ac.uk, zhang.wen@zju.edu.cn}
  \country{}}

\renewcommand{\shortauthors}{Yang et al.}

\begin{abstract}
OWL ontologies provide  a formal knowledge representation framework that enables semantic reasoning, and have been widely adopted across domains such as healthcare and bioinformatics. 
In practice, however, real-world ontologies are often incomplete, which poses challenges for reasoning. 
In this work, we focus on a fundamental subsumption reasoning problem: given an incomplete ontology and a candidate (non-entailed) subsumption, determine whether the subsumption is semantically plausible and, if so, provide a logically sound explanation containing potential missing axioms. 
To address this subsumption reasoning problem, we propose \ours, an end-to-end neural-symbolic framework that jointly performs verification and abduction{ through a hierarchical, structure-guided search for bridge concepts}. Its key contribution is the ability to automatically identify meaningful missing components that are grounded in existing information and plausible in real-world contexts, without requiring a predefined hypothesis scope.
We evaluate \ours on three real-world ontologies across multiple domains, demonstrating strong and robust performance.
\end{abstract}

\maketitle

\section{Introduction}\label{sec:intro}
\begin{figure}[t]
\centering
\providecommand{\cn}[1]{\texttt{#1}}
\definecolor{tInk}{HTML}{262626}
\definecolor{tMuted}{HTML}{5E6670}
\definecolor{tFill}{HTML}{F3F5F7}
\definecolor{tBlueFill}{HTML}{E8EEF4}
\definecolor{tNode}{HTML}{8C949D}
\definecolor{tBridge}{HTML}{5A7FA6}
\definecolor{tDark}{HTML}{263F5C}
\definecolor{tMiss}{HTML}{B23A32}
\definecolor{tBad}{HTML}{9A4A3F}
\definecolor{tBadFill}{HTML}{F6ECEA}
\definecolor{tGrid}{HTML}{E4E4E4}
\definecolor{tGood}{HTML}{2E7D4F}
\definecolor{tGoodFill}{HTML}{E6F2EA}
\begin{tikzpicture}[x=1pt, y=1pt, font=\small, text=tInk,
  cnode/.style={draw=tNode, line width=0.5pt, rounded corners=2pt, fill=white, inner xsep=2.4pt, inner ysep=1.8pt, font=\small},
  sub/.style={-{Stealth[length=3pt]}, tMuted, line width=0.6pt},
  miss/.style={-{Stealth[length=3pt]}, tMiss, line width=0.75pt, dash pattern=on 2.2pt off 1.4pt}]
\fill[tBlueFill, rounded corners=2.5pt] (0,114.80000000000001) rectangle (95.0,142.8);
\fill[tBridge] (0,114.80000000000001) rectangle (1.8,142.8);
\node[anchor=base west, inner sep=0pt, font=\small\bfseries] at (6,131.8) {Query $\alpha$};
\node[anchor=base west, inner sep=0pt] at (6,121.20000000000002) {\cn{PetCat} $\sqsubseteq$ \cn{Animal}\,?};
\node[anchor=base west, inner sep=0pt, font=\footnotesize, text=tMuted] at (1,105.30000000000001) {incomplete ontology $\mathcal{O}'_1$:};
\node[cnode] (pt) at (47.5,5.6) {\cn{PetCat}};
\node[cnode] (ca) at (47.5,35.17) {\cn{Cat}};
\node[cnode] (ma) at (47.5,64.73) {\cn{Mammal}};
\node[cnode, anchor=west] (an) at (0,94.30000000000001) {\cn{Animal}};
\node[cnode, anchor=east] (mi) at (95.0,94.30000000000001) {$\exists\cn{produces}.\cn{Milk}$};
\draw[sub] (ca) -- (ma);
\draw[sub] (ma) -- (an);
\draw[sub] (ma) -- (mi);
\draw[miss] (pt) -- node[right=1pt, font=\footnotesize\itshape, text=tMiss] {missing} (ca);
\node[circle, fill=tBadFill, draw=tBad, line width=0.5pt, text=tBad, inner sep=0pt, minimum size=9.6pt, font=\scriptsize\bfseries] at (105.5,134.8) {$\times$};
\node[anchor=base west, inner sep=0pt, font=\small\bfseries] at (112.0,131.8) {OWL reasoner};
\node[anchor=base east, inner sep=0pt] at (237.0,131.8) {\textcolor{tBad}{\textbf{No}}};
\fill[tNode] (105.5,123.20) circle (1.1pt);
\node[anchor=base west, inner sep=0pt] at (112.0,121.2) {$\alpha$ is not entailed by $\mathcal{O}'_1$};
\draw[tGrid, line width=0.45pt] (100.0,113.7) -- (241.0,113.7);
\node[circle, fill=tFill, draw=tMuted, line width=0.5pt, text=tMuted, inner sep=0pt, minimum size=9.6pt, font=\scriptsize\bfseries] at (105.5,103.2) {?};
\node[anchor=base west, inner sep=0pt, font=\small\bfseries] at (112.0,100.2) {TBox abduction};
\node[anchor=base east, inner sep=0pt] at (237.0,100.2) {\textcolor{tMuted}{\textit{predefined scope}}};
\fill[tNode] (105.5,91.60) circle (1.1pt);
\node[anchor=base west, inner sep=0pt] at (112.0,89.6) {finds \textcolor{tMiss}{$\cn{PetCat}\sqsubseteq\cn{Cat}$} only if in scope};
\fill[tNode] (105.5,81.00) circle (1.1pt);
\node[anchor=base west, inner sep=0pt] at (112.0,79.0) {pure logic: no textual information};
\fill[tNode] (105.5,70.40) circle (1.1pt);
\node[anchor=base west, inner sep=0pt] at (112.0,68.4) {assumes $\alpha$ holds: never \textbf{No}};
\fill[tGoodFill, rounded corners=2.5pt] (98.0,0.0) rectangle (241.0,58.4);
\fill[tGood] (98.0,0.0) rectangle (99.8,58.4);
\node[circle, fill=tGoodFill, draw=tGood, line width=0.5pt, text=tGood, inner sep=0pt, minimum size=9.6pt, font=\scriptsize\bfseries] at (105.5,50.4) {$\checkmark$};
\node[anchor=base west, inner sep=0pt, font=\small\bfseries] at (112.0,47.4) {\ours};
\node[anchor=base east, inner sep=0pt] at (237.0,47.4) {\textcolor{tGood}{\textbf{Yes}}};
\fill[tGood] (105.5,38.80) circle (1.1pt);
\node[anchor=base west, inner sep=0pt] at (112.0,36.8) {completes $\mathcal{O}'_1$ with \textcolor{tMiss}{$\cn{PetCat}\sqsubseteq\cn{Cat}$}};
\fill[tGood] (105.5,28.20) circle (1.1pt);
\node[anchor=base west, inner sep=0pt] at (112.0,26.2) {no predefined scope needed};
\fill[tGood] (105.5,17.60) circle (1.1pt);
\node[anchor=base west, inner sep=0pt] at (112.0,15.6) {integrates textual information};
\fill[tGood] (105.5,7.00) circle (1.1pt);
\node[anchor=base west, inner sep=0pt] at (112.0,5.0) {implausible queries: \textbf{No}};
\end{tikzpicture}
\caption{
Comparison with reasoning and abduction: reasoners struggle with incomplete ontologies, while abduction requires predefined hypotheses and assumed valid conclusions. \ours{} overcomes this by integrating textual and ontological information via LLM reasoning.
}
\label{fig:teaser}
\end{figure}

\begin{figure*}[t]
\centering
\definecolor{ovInk}{HTML}{2E3338}
\definecolor{ovMuted}{HTML}{5E6670}
\definecolor{ovFrame}{HTML}{C3C9D0}
\definecolor{ovFill}{HTML}{F3F5F7}
\definecolor{ovNode}{HTML}{8C949D}
\definecolor{ovBridge}{HTML}{5A7FA6}
\definecolor{ovMiss}{HTML}{B23A32}
\definecolor{ovCard}{HTML}{F6F8FA}
\definecolor{ovHead}{HTML}{E9EEF3}
\definecolor{ovShadow}{HTML}{8C949D}
\definecolor{ovTrue}{HTML}{E6F2EA}
\definecolor{ovGood}{HTML}{2E7D4F}
\definecolor{ovDark}{HTML}{263F5C}
\definecolor{ovFalse}{HTML}{F6ECEA}
\definecolor{ovBad}{HTML}{9A4A3F}
\begin{tikzpicture}[
  x=1cm, y=1cm, font=\small, text=ovInk,
  execute at begin picture={\thickmuskip=3mu plus 1mu\relax},
  card/.style={fill=white, draw=ovFrame, line width=0.55pt, rounded corners=4pt, drop shadow={shadow xshift=0.6pt, shadow yshift=-0.7pt, opacity=0.25, fill=ovShadow}},
  tag/.style={fill=white, draw=ovFrame, line width=0.45pt, rounded corners=2pt, inner xsep=5pt, inner ysep=1.3pt, font=\footnotesize\scshape, text=ovMuted},
  res/.style={fill=ovFill, draw=ovFrame, line width=0.5pt, rounded corners=4pt},
  term/.style={rounded corners=4pt, minimum height=0.62cm, inner xsep=7pt, inner ysep=0pt, drop shadow={shadow xshift=0.6pt, shadow yshift=-0.7pt, opacity=0.25, fill=ovShadow},
    path picture={\fill[#1] (path picture bounding box.north west) rectangle ([xshift=2.2pt]path picture bounding box.south west);}},
  ttl/.style={anchor=base, inner sep=0pt},
  cn/.style={circle, draw=ovNode, line width=0.5pt, fill=white, minimum size=0.42cm, inner sep=0pt},
  br/.style={cn, draw=ovBridge, line width=0.6pt, fill=ovBridge!12},
  ent/.style={-{Stealth[length=3pt,width=2.8pt]}, line width=0.6pt, ovMuted},
  mis/.style={-{Stealth[length=3pt,width=2.8pt]}, line width=0.65pt, ovMiss, dash pattern=on 1.8pt off 1.2pt},
  el/.style={anchor=mid west, inner sep=0pt},
  flow/.style={-{Stealth[length=4pt,width=3.6pt]}, line width=0.7pt, ovMuted},
  fl/.style={font=\small\itshape, text=ovMuted, inner sep=0pt, anchor=base, yshift=3.5pt},
  yl/.style={font=\small\itshape, text=ovMuted, inner sep=0pt, anchor=mid west, xshift=3pt},
]
\foreach \xa/\xb/\t in {1.62/4.02/Stage 1, 4.52/10.02/Stage 2, 10.52/16.02/Stage 3} {
  \draw[card] (\xa,0) rectangle (\xb,-2.72);
  \node[tag] at ({(\xa+\xb)/2},0) {\t};
}
\foreach \xa/\xb in {1.72/3.92, 4.62/7.22, 7.32/9.92, 10.62/13.22, 13.32/15.92}
  \fill[ovHead, rounded corners=2pt] (\xa,-0.28) rectangle (\xb,-0.70);
\node[ttl] at (2.82,-0.53)   {\textbf{1}\enspace Reasoner};
\node[ttl] at (5.895,-0.53)  {\textbf{2a}\enspace Downward};
\node[ttl] at (8.645,-0.53)  {\textbf{2b}\enspace Upward};
\node[ttl] at (11.895,-0.53) {\textbf{3a}\enspace Bidirectional};
\node[ttl] at (14.645,-0.53) {\textbf{3b}\enspace Direct check};
\node[inner sep=0pt] at (2.82,-1.68) {$\mathcal{O}'\models A\sqsubseteq B$\,?};
\node[cn] (aB) at (5.335,-1.02) {$B$};
\node[br] (aM) at (5.335,-1.68) {$B'$};
\node[cn] (aA) at (5.335,-2.34) {$A$};
\draw[ent] (aM) -- (aB);  \node[el] at (5.665,-1.35) {$B'\sqsubseteq B$};
\draw[mis] (aA) -- (aM);  \node[el, text=ovMiss] at (5.665,-2.01) {$A\sqsubseteq B'$};

\node[cn] (bB) at (8.085,-1.02) {$B$};
\node[br] (bM) at (8.085,-1.68) {$A'$};
\node[cn] (bA) at (8.085,-2.34) {$A$};
\draw[mis] (bM) -- (bB);  \node[el, text=ovMiss] at (8.415,-1.35) {$A'\sqsubseteq B$};
\draw[ent] (bA) -- (bM);  \node[el] at (8.415,-2.01) {$A\sqsubseteq A'$};

\node[cn] (cB) at (11.335,-1.02) {$B$};
\node[br] (cM) at (11.335,-1.68) {$C$};
\node[cn] (cA) at (11.335,-2.34) {$A$};
\draw[mis] (cM) -- (cB);  \node[el, text=ovMiss] at (11.665,-1.35) {$C\sqsubseteq B$};
\draw[mis] (cA) -- (cM);  \node[el, text=ovMiss] at (11.665,-2.01) {$A\sqsubseteq C$};

\node[cn] (dB) at (14.085,-1.02) {$B$};
\node[cn] (dA) at (14.085,-2.34) {$A$};
\draw[mis] (dA) -- (dB);  \node[el, text=ovMiss] at (14.415,-1.68) {$A\sqsubseteq B$};
\node[term=ovNode, fill=ovFill] (IN) at (0.64,-1.68) {$A\sqsubseteq B$\,?};
\draw[flow] (IN.east)     -- (1.62,-1.68);
\draw[flow] (4.02,-1.68)  -- node[fl]{no}   (4.52,-1.68);
\draw[flow] (6.75,-1.68)  -- node[fl]{fail} (7.80,-1.68);
\draw[flow] (10.02,-1.68) -- node[fl]{no}   (10.52,-1.68);
\draw[flow] (12.75,-1.68) -- node[fl]{fail} (13.80,-1.68);
\node[term=ovBad, fill=ovFalse, text=ovBad, font=\small\scshape] (F) at (17.12,-1.68) {False};
\draw[flow] (16.02,-1.68) -- node[fl]{no} (F.west);
\node[term=ovGood, fill=ovTrue, inner sep=0pt, fit={(IN.west |- 0,-3.20) (16.02,-3.82)}] {};
\node[font=\small\scshape, text=ovGood] at (0.64,-3.51) {True};
\node at (2.82,-3.51)   {$J_{A\sqsubseteq B}$};
\node at (5.895,-3.51)  {$\{\textcolor{ovMiss}{A\sqsubseteq B'}\}\cup J_{B'\sqsubseteq B}$};
\node at (8.645,-3.51)  {$\{\textcolor{ovMiss}{A'\sqsubseteq B}\}\cup J_{A\sqsubseteq A'}$};
\node at (11.895,-3.51) {$\{\textcolor{ovMiss}{A\sqsubseteq C},\,\textcolor{ovMiss}{C\sqsubseteq B}\}$};
\node at (14.645,-3.51) {$\{\textcolor{ovMiss}{A\sqsubseteq B}\}$};
\foreach \x in {2.82,5.895,8.645,11.895,14.645}
  \draw[flow] (\x,-2.72) -- node[yl]{yes} (\x,-3.20);
\end{tikzpicture}
\caption{Overview of \ours:
Red axioms denote potentially missing axioms.
}
\label{fig:architecture}
\end{figure*}

\begin{sloppypar}Ontologies expressed in the Web Ontology Language (OWL) provide a formal framework for representing structured knowledge and enable automated reasoning through description logics (DLs)~\cite{baader2003description}. Owing to their precise and shareable representation of machine-interpretable semantics, OWL ontologies are widely used in knowledge-intensive domains such as biomedicine~\cite{ashburner2000gene} and healthcare~\cite{SNOMED}.\end{sloppypar}

A fundamental reasoning task in ontology engineering is subsumption reasoning: determining whether a subsumption relationship holds between two concepts and providing explanations for such entailments. For complete and consistent ontologies, subsumption can be determined by standard reasoners~\cite{glimm2014hermit,kazakov2012elk}, and explanations can be derived as \emph{justifications} (\textit{i.e.}, minimal sets of axioms sufficient for the entailment)~\cite{penaloza2020axiom,PULi,yang2022hypergraph,alrabbaa2024explaining}. 
In practice, however, real-world ontologies are often incomplete due to various reasons{, such as} evolving domain knowledge and the high cost of expert curation. As a result, some valid subsumptions may not be entailed because some required axioms are missing.

\emph{Abductive reasoning} has been proposed as a natural solution to this problem~\cite{du2017practical,haifani2022connection,wei2014abduction,del2019abox,du2012towards,du2014tractable,halland2012abox,klarman2011abox,koopmann2021signature,pukancova2020aaa,DBLP:conf/owled/ElsenbroichKS06,koopmann2020signature,bienvenu2008complexity,haak2025not}. 
Among the different forms of abduction, \emph{TBox abduction} is the most relevant to our setting. It aims to find a set of hypotheses (missing axioms) $\mathcal{H}$ such that, when added to an incomplete ontology $\mathcal{O}'$, the target subsumption $\alpha$ becomes entailed (\textit{i.e.}, $\mathcal{O}' \cup \mathcal{H} \models \alpha$). 
Despite their promise, existing approaches exhibit several limitations: (1) they typically rely on a predefined hypothesis scope, such as restricting candidate axioms to a fixed vocabulary~\cite{haifani2022connection, koopmann2020signature}; (2) they often assume that the target subsumption is valid, limiting their applicability in more general settings where the validity of $\alpha$ is itself uncertain. Figure~\ref{fig:teaser} illustrates these limitations on a small example.
Recently, there has been growing interest in using large language models (LLMs) for related ontology reasoning and construction tasks such as subsumption prediction and concept placement~\cite{ontolama, chen2023contextual,hit,OnT,shi2024taxonomy,babaei2023llms4ol, zhang2025ontourl,lo2024end,DBLP:journals/pvldb/SunHSXYDTC24}. 
However, these works primarily focus on predicting or verifying ontological relationships rather than generating abductive explanations.

To address these challenges, we propose \ours, an \textbf{end-to-end} neural-symbolic framework for explainable subsumption reasoning over incomplete OWL ontologies. 
 \ours\ supports both positive and negative entailment scenarios and does \textbf{not} require a predefined hypothesis space, such as a query-specific list of candidate missing axioms. Instead, it automatically identifies meaningful missing components that are both grounded in the existing information and plausible in real-world contexts, by integrating textual information with formal semantics through LLMs and ontology embeddings.
\ours operates in both training-free and fine-tuned settings, making it adaptable across domains.

The main idea of \ours, as illustrated in Figure~\ref{fig:architecture}, is to uncover missing knowledge by identifying intermediate concepts that connect related notions. In the example of Figure~\ref{fig:teaser} (Example~\ref{exp:main} gives the full version), $\texttt{PetCat} \sqsubseteq \texttt{Animal}$ is expected to hold but cannot be derived from the incomplete $\Omc_1'$.
To address this, \ours\ searches for intermediate concepts $C$ such that $\texttt{PetCat} \sqsubseteq C \sqsubseteq \texttt{Animal}$, which can be used to construct missing axioms and explanations. 
Depending on the scenario, part or all of the connections may be missing (cf. Stages~2a/2b and~3a in Figure~\ref{fig:architecture}).
Here, \ours\ selects $C = \texttt{Cat}$: the missing axiom $\texttt{PetCat} \sqsubseteq \texttt{Cat}$ and the axioms of $\Omc_1'$ that entail $\texttt{Cat} \sqsubseteq \texttt{Animal}$ form the explanation.
Candidate intermediate concepts are verified by LLMs, while ontology embeddings~\cite{OnT} provide a scalable mechanism for filtering and shortlisting candidates by semantic similarity. 
If no suitable intermediate concept is found, \ours\ treats the target subsumption as a simple, indivisible axiom, proposes it directly, and validates it using an LLM (Stage~3b in Figure~\ref{fig:architecture}).

We evaluate \ours\ on multiple real-world ontologies across diverse domains. Experimental results demonstrate that our approach achieves strong performance in concept subsumption reasoning while also generating meaningful explanations for incomplete input ontologies. 
Our results show that, on datasets constructed from the real-world ontologies FoodOn~\cite{dooley2018foodon}, Gene Ontology~\cite{ashburner2000gene} and  Snomed CT~\cite{SNOMED}, \ours\ achieves high performance in checking subsumption plausibility, with F1 scores of up to 0.970. 
They also demonstrate high accuracy of the generated explanations, as reflected in the correctness of predicted missing axioms, achieving X-F1 and X-F1$^*$ scores of up to 0.894 and 0.793, respectively. Further details are provided in Section~\ref{sec:results}.

The main contributions of this work are:
\begin{itemize}[leftmargin=*]
    \item We introduce a generalized TBox abduction task that does not require predefined candidate missing axioms and handles cases where the given subsumption is incorrect.
    \item We propose \ours, an end-to-end neural-symbolic framework for reasoning over incomplete ontologies, utilizing LLMs and ontology embeddings.
    \item We empirically evaluate our approach on real-world ontologies, demonstrating its effectiveness and robustness on five datasets with different kinds of missing axioms.\end{itemize}

The prompts, details of the datasets and additional results are given in the appendix.

\section{Related Work}
\label{sec:related}

\textbf{Abductive Reasoning} on ontologies can be divided into several categories, including  \textit{TBox abduction}~\cite{du2017practical,haifani2022connection,wei2014abduction}, \textit{ABox abduction}~\cite{del2019abox,du2012towards,du2014tractable,halland2012abox,klarman2011abox,koopmann2021signature,pukancova2020aaa,haak2025not},
\textit{knowledge-base abduction}~\cite{DBLP:conf/owled/ElsenbroichKS06,koopmann2020signature}, and \textit{concept abduction}~\cite{bienvenu2008complexity}. These categories are distinguished by the types of the missing elements: rules (TBox axioms), facts (ABox axioms),  a combination of both, and concepts, respectively. In this work, we focus on TBox abduction, where the missing elements are TBox axioms, as defined in Section~\ref{sec:intro}.

Existing approaches to TBox abduction rely on different assumptions regarding the selection of a predefined set of candidate axioms. For example, \citet{wei2014abduction} assume a human-verified selection defined by an oracle function provided by domain experts; \citet{du2017practical} rely on a set of plausible patterns; and \citet{haifani2022connection,koopmann2020signature} consider only axioms formulated over a given set of concepts and roles, known as a signature. These assumptions limit the practicality of such methods in real-world applications. Moreover, these approaches focus on logical computation and largely ignore the use of informal meta information like textual labels{,} which also contain important semantics.

\textbf{Ontology Embeddings} aim to encode ontology entities (\textit{i.e.}, concepts, roles, and instances) as numerical vectors while preserving their structural and semantic properties. These embeddings support several downstream machine learning and data mining tasks including concept subsumption prediction~\cite{chen2024ontology}. 
Existing approaches can be broadly divided into two categories: (1) \textit{Geometric model-based methods}~\cite{kulmanov_embeddings_2019,xiong_faithiful_2022,DBLP:conf/www/JackermeierC024,yang2025transbox} represent ontology entities as geometric objects, translating DL operators into geometric operations. For example, concept subsumption and intersection are modeled as region inclusion and intersection, respectively.
(2)  \textit{Language-model-based methods}~\cite{DBLP:journals/bioinformatics/SmailiGH19,DBLP:journals/ml/ChenHJHAH21, kulmanov2021semantic} leverage both formal knowledge and textual information to capture both linguistic and semantic features in the embedding space. 
More recently introduced, HiT~\cite{hit} combines language models with hierarchical embeddings in hyperbolic space to represent taxonomies and support direct subsumption inference between two arbitrary concept labels. Building on this, OnT~\cite{OnT} extends the approach to $\mathcal{EL}$ ontologies by introducing role embeddings as rotations and additional loss functions to encode logical properties. In this work, we adopt OnT as our main embedding method.

\textbf{LLMs for Ontology Subsumption} have been investigated in many tasks such as subsumption prediction or inference~\cite{ontolama, zhang2025ontourl, chen2023contextual, hit, OnT}, taxonomy completion~\cite{shi2024taxonomy}, axiom learning~\cite{babaei2023llms4ol, zhang2025ontourl, lo2024end}, and the intrinsic knowledge of hierarchical structures in LLMs~\cite{DBLP:journals/pvldb/SunHSXYDTC24}. 
However, to the best of our knowledge, few studies have specifically explored the use of LLMs for assessing the semantic plausibility of subsumptions. 
\citet{zhao2026subsumption} propose a related approach but assume additional verification by human experts.
Moreover, most of these LLM-based works, as well as the above introduced ontology embedding-based works, focus on the correctness of the subsumption, and little attention has been devoted to explanations. 
\citet{huiWWW26} investigate explanation generation and consider incomplete settings, but assume the correctness of the given subsumption and focus on evaluating the capabilities of LLMs instead of developing an end-to-end reasoner.

\section{Preliminaries}
\label{sec:prelim}

Ontologies provide a formal framework for representing structured knowledge through collections of logical statements, known as \emph{axioms}. These axioms describe relationships between \emph{concepts} (unary predicates) and \emph{roles} (binary predicates).  
Here, we take $\mathcal{EL}$ ontologies as an example{; $\mathcal{EL}$ is a lightweight yet expressive description logic} that has been successfully employed in areas such as life sciences and semantic data integration~\cite{DBLP:books/daglib/0041477}.  
Let $\NC = \{A, B, \ldots\}$, $\NR = \{r, s, \ldots\}$ denote two mutually disjoint sets of \emph{concept names} and \emph{role names}, respectively.  
The set of \emph{$\mathcal{EL}$-concepts} is inductively defined as follows:
\(
C ::= \top \mid A \mid C \sqcap D \mid \exists r.C,
\)
where $A \in \NC$, $r \in \NR$, and $C, D$ are $\mathcal{EL}$-concepts.  
In this work, the target subsumptions and the missing axioms are TBox axioms, also called \emph{General Concept Inclusions} (GCIs), of the form $C \sqsubseteq D$ with $\mathcal{EL}$-concepts $C$ and $D$, while the ontology itself may use any OWL constructs. 

\begin{example}\label{exp0}
Given the concept names 
\texttt{PersianCat},\ \texttt{PetCat},\ \texttt{Cat},\ \texttt{Mammal},\ \texttt{Animal}, \texttt{Milk}
and the role $\texttt{produces}$, we may define a simple ontology $\mathcal{O}_1'$ as follows:
\[
\begin{aligned}
\{\, &\texttt{PersianCat} \sqsubseteq \texttt{PetCat}, \quad
\texttt{Cat} \sqsubseteq \texttt{Mammal},\\
&\texttt{Mammal} \sqsubseteq \texttt{Animal} \sqcap
\exists\,\texttt{produces}.\texttt{Milk} \,\}.
\end{aligned}
\]
\end{example}

\noindent
Logical reasoning over an ontology $\Omc$ derives new knowledge from its axioms. Formally, entailment ($\models$) is defined via interpretations and models (\textit{e.g.}, from $A \sqsubseteq B$ and $B \sqsubseteq C \in \Omc$, it follows that $\Omc \models A \sqsubseteq C$; see~\citet{DBLP:books/daglib/0041477} for more details). To explain an entailment, one can examine its \emph{justifications}: minimal sets of axioms that entail it~\cite{penaloza2020axiom}.

\begin{definition}[Justification]\label{defi:justfication}
A \emph{justification} for $A \sqsubseteq B$ is a subset $J_{A\sqsubseteq B} \subseteq \mathcal{O}$ that is minimal with respect to set inclusion and satisfies $J_{A\sqsubseteq B} \models A \sqsubseteq B$.
\end{definition}

\section{The \ours System}
\label{sec:method}

The task is to determine whether a given concept inclusion $\alpha = A \sqsubseteq B$ is plausible with respect to a possibly incomplete ontology $\Omc'$ and, if so, to return a set of axioms $E$ as a possible explanation, such that $E \models \alpha$. The explanation $E$ may include missing axioms that are not present in $\Omc'$.
To ensure that explanations rely primarily on the stated axioms in $\Omc'$, we restrict the method to returning at most two missing concept inclusion axioms per iteration, whenever possible. If additional missing axioms are required, the proposed method can be applied iteratively, as discussed at the end of this section.

Without loss of generality, we focus on subsumptions of the form \(A \sqsubseteq B\) between atomic concepts \(A\) and \(B\). The framework can be naturally extended to subsumptions between complex concepts, \(C \sqsubseteq D\). For instance, one possible approach is to introduce fresh concept names for the complex concepts, such as \(A \equiv C\) and \(B \equiv D\), and reduce the problem to the corresponding atomic subsumption \(A \sqsubseteq B\).

\subsection{Main Structure}
The overall \ours framework is illustrated in Figure~\ref{fig:architecture}. It integrates logical structure with neural verification to predict missing axioms and consists of three stages:
\begin{itemize}[leftmargin=*]
    \item \textit{Stage 1. Logical Checking}: Determine whether the input query can be logically entailed by $\mathcal{O}'$ using a standard ontology reasoner.
    \item \textit{Stage 2. Logical Bridging}: Identify potentially missing axioms using candidate concepts induced by the logical structure of $\mathcal{O}'$, combined with LLM-based validation.
    \item \textit{Stage 3. Bridging by Embedding}: Generate candidates directly based on semantic similarity using embeddings when the previous stages fail.
\end{itemize}

\paragraph{\textbf{Stage 1} (Logical Checking)}
In Stage 1,
we use a standard description logic reasoner to test whether
\(
\mathcal{O}' \models A \sqsubseteq B
\)
holds. If so, the system returns \textsc{True} with a justification derived from $\mathcal{O}'$. Otherwise, \ours assumes some axioms are missing and proceeds to the subsequent stages.

\paragraph{\textbf{Stage 2} (Logical Bridging)}
In Stage 2, we attempt to identify \emph{missing bridge concepts} using the logical information in $\mathcal{O}'$. There are two cases:
\begin{itemize}[leftmargin=*]
    \item \textit{Stage~2a (Downward):} Consider each child $B'$ of $B$ in $\mathcal{O}'$ (\textit{i.e.}, $\mathcal{O}' \models B' \sqsubseteq B$) as a candidate bridge concept. Then, $A \sqsubseteq B'$ is treated as a candidate missing axiom and will be verified by the LLM.

    \item \textit{Stage~2b (Upward):} Consider each parent $A'$ of $A$ in $\mathcal{O}'$ (\textit{i.e.}, $\mathcal{O}' \models A \sqsubseteq A'$) as a candidate, and use the LLM to validate the corresponding potential missing axiom $A' \sqsubseteq B$.
\end{itemize}
Each missing axiom candidate is then verified by an LLM via prompting. 
If a valid bridge is found, the system returns \textsc{True} along with the constructed explanation. 
For example, if $A \sqsubseteq B'$ is validated, the explanation is:
\(
\{A \sqsubseteq B'\}\cup J_{B' \sqsubseteq B}.
\)

Note that LLMs may produce redundant answers that can be pruned directly. For instance, if $B_1'$ and $B_2'$ are selected in Stage 2a with $\Omc' \models B_1' \sqsubseteq B_2'$, it suffices to retain $B_1'$ with the missing axiom $A \sqsubseteq B_1'$, since $A \sqsubseteq B_2'$ follows from $\Omc' \cup \{A \sqsubseteq B_1'\}$. Example~\ref{exp:main} provides a more concrete instance.

In our implementation, instead of directly calling the LLM to check each potential missing axiom, \ours first uses an ontology embedding model (\textit{e.g.}, OnT~\cite{OnT}) to retrieve the top-$k$ candidates, which are ranked by the subsumption score defined by the embedding model.
This not only bounds the number of candidates that the LLM has to verify, but also exploits the semantics captured by the embeddings.

\paragraph{\textbf{Stage 3} (Bridging by Embedding)}
In Stage 3,  
if no suitable candidates are found in the previous stage, we search for missing axioms based purely on semantic information.

\begin{itemize}[leftmargin=*]
    \item \textit{Stage 3a (Bidirectional search):} We first perform a bidirectional search (Stage~3a in Figure~\ref{fig:architecture}) for an intermediate concept $C$ such that
\(
A \sqsubseteq C \sqsubseteq B,
\)
where $C$ is, by default, drawn from all atomic concepts in $\Omc'$ (excluding $A$ and $B$), but can be customized based on user interest (\textit{e.g.}, existential restrictions of the form $\exists s.X$).
Given the large candidate space (potentially hundreds of thousands of atomic concepts or millions of restrictions $\exists s.X$), exhaustive LLM verification is computationally prohibitive. We therefore use ontology embeddings to rank candidates and retain only the top-$k$ for LLM verification. Candidates are ranked according to
\[
\bar{s}(C) = \frac{1}{2}\big(s(A \sqsubseteq C) + s(C \sqsubseteq B)\big),
\]
where $s(A \sqsubseteq C)$ denotes the plausibility score assigned by the embedding model to the subsumption $A \sqsubseteq C$.
If a valid $C$ is identified, the system returns \textsc{True} with the explanation
\[
\{A \sqsubseteq C,\; C \sqsubseteq B\}.
\]
Otherwise, the subsumption \(A \sqsubseteq B\) is treated as an \emph{elementary axiom} (\textit{i.e.}, requiring no further decomposition) and as a potential missing axiom itself. 

\item \textit{Stage 3b (Direct check):} Then, \ours uses the LLM to check the correctness of the subsumption \(A \sqsubseteq B\) directly in Stage~3b, as shown in Figure~\ref{fig:architecture}.
If the LLM judges \(A \sqsubseteq B\) to be correct, \ours returns \textsc{True} along with the subsumption itself as the missing axiom and explanation. Otherwise, \ours concludes that the input subsumption does not plausibly hold and returns \textsc{False}.
\end{itemize}

\begin{example}[Continuation of Example~\ref{exp0}]\label{exp:main}
The ontology $\mathcal{O}_1'$ given in Example \ref{exp0} is incomplete and lacks the axiom:
\(
\texttt{PetCat} \sqsubseteq \texttt{Cat}.
\)
Thus, it cannot entail the following  subsumption: 
\(
\texttt{PetCat} \sqsubseteq \texttt{Animal}.
\)
If we input this subsumption into \ours, then we have:
\begin{itemize}[leftmargin=*]
    \item \textbf{Stage 1:} 
    The reasoner cannot derive the given subsumption due to the incompleteness of $\mathcal{O}_1'$.

    \item \textbf{Stage 2a (Downward):} 
    The system explores subclasses of $\texttt{Animal}$; in this simple case, only $\texttt{Mammal}$ and $\texttt{Cat}$ are considered,
    so no filtering step is required.
    This yields two candidate missing axioms:
    \[
    \texttt{PetCat} \sqsubseteq \texttt{Cat}, \quad \texttt{PetCat} \sqsubseteq \texttt{Mammal}.
    \]
    An LLM is then used to select the axioms that are likely to be valid. If the LLM does not select any candidates, the process proceeds to the next stage.

    \item \textbf{Return:} 
    Suppose the LLM identifies both candidate axioms as valid. In this case, \ours concludes that the queried subsumption holds, yielding two outputs. 
    Note that $\texttt{PetCat} \sqsubseteq \texttt{Mammal}$ can be ignored, as it is derivable from $\Omc_1'$ together with $\texttt{PetCat} \sqsubseteq \texttt{Cat}$. Finally, we output the missing axiom
    \(
    \texttt{PetCat} \sqsubseteq \texttt{Cat},
    \)
    and the corresponding explanation:
    \[
    \begin{gathered}\{\texttt{PetCat} \sqsubseteq \texttt{Cat}\} \;\cup\; \{\texttt{Cat} \sqsubseteq \texttt{Mammal},\\ \texttt{Mammal} \sqsubseteq \texttt{Animal}\sqcap\exists\texttt{produces}.\texttt{Milk}\},\end{gathered}
    \]
    where the latter forms a justification for $\mathcal{O}_1' \models \texttt{Cat} \sqsubseteq \texttt{Animal}$. 

\end{itemize}

\end{example}

It is worth noting that \ours\ can be applied iteratively. By repeatedly feeding the predicted missing axioms back into the framework, \ours\ can propose further, increasingly fine-grained missing axioms, rather than being limited to at most two in the single-pass setting. A corresponding case study is presented in Section~\ref{sec:case_study}.

\section{Dataset Construction}
\label{sec:data}
We construct datasets from real-world ontologies to evaluate our method in two settings: atomic and existential. We construct an incomplete ontology $\mathcal{O}'$ by removing or simplifying axioms from the original ontology $\mathcal{O}$, which serves as the \emph{reference ontology} for determining sample labels through logical entailment.

\subsection{Atomic Setting} In this setting, we use all atomic concepts as the candidate set of bridge concepts $C$ 
(\textit{i.e.}, $A\sqsubseteq C\sqsubseteq B$ for a given subsumption $A\sqsubseteq B$) of \ours. The dataset is constructed as follows.

First, to simulate ontology incompleteness in a controlled manner, we construct a pruned ontology $\mathcal{O}'$ from the original ontology $\mathcal{O}$ by uniformly sampling and removing 5\% of the TBox axioms.
Then, the evaluation dataset with positive and negative samples of subsumptions is constructed as follows: 
\begin{enumerate}[leftmargin=*]
    \item \textbf{Positive samples.}  
A positive sample is a subsumption  $A\sqsubseteq B$ such that $\mathcal{O} \models A \sqsubseteq B$ but $\mathcal{O}' \not\models A \sqsubseteq B$. We construct the positive sample set $\mathcal{P} $ by randomly sampling from all positive candidates. That is,
\[
\mathcal{P} \subseteq \mathcal{P}_{\text{all}}:= \{(A, B) \mid \mathcal{O} \models A \sqsubseteq B \ \wedge\ \mathcal{O}' \not\models A \sqsubseteq B \}.
\]
These positive samples may involve one or several bridge concepts, corresponding to downward, upward, or bidirectional cases associated with Stages 2a, 2b, and 3a, respectively. In such cases, the induced missing axioms, together with associated justifications for Stages 2a and 2b, serve as the ground-truth explanations. Alternatively, they may involve no bridge and are therefore expected to be resolved in Stage 3b; in this case, the subsumptions themselves serve as the ground-truth explanation.
\item \textbf{Negative samples.}  
Negative samples are subsumptions $A\sqsubseteq B$ such that $\mathcal{O} \not\models A \sqsubseteq B$. To construct a balanced dataset, we sample an equal number of negative samples (\textit{i.e.}, $|\mathcal{N}| = |\mathcal{P}|$) by replacing $B$ with a $B'$ satisfying $\mathcal{O} \not\models A \sqsubseteq B'$, using the following two strategies:
\begin{itemize}[leftmargin=*]
    \item \textbf{Random:} Randomly choose one $B'$ from all candidates.
    \item \textbf{Hard:} Restrict the candidates $B'$ to those that are close to $B$, in the sense that there exists a path of length at most $4$ between $B$ and $B'$ in the taxonomy of $\mathcal{O}$ (ignoring the edge direction)\footnote{The taxonomy of $\mathcal{O}$ is a directed acyclic graph over its concepts, with an edge $A \to B$ iff $\mathcal{O} \models A \sqsubseteq B$ and no $C \neq A,B$ satisfies $\mathcal{O} \models A \sqsubseteq C \sqsubseteq B$.}, preferring the closest ones. If none exists, then choose $B'$ randomly.
\end{itemize}

\end{enumerate}

\subsection{Existential Setting}\label{sec:data_ex}
In the atomic setting, \ours{} always returns missing axioms in the form of subsumptions between atomic concepts. In some cases, however, a missing axiom may involve a complex concept. For example, given the ontology \(\{A \sqsubseteq A_1,\; B_1 \sqsubseteq B_2,\; \exists r.B_2 \sqsubseteq B\}\), a plausible missing axiom is \(A \sqsubseteq \exists r.B_1\), whose addition enables the derivation of \(A \sqsubseteq B\).

We therefore consider an existential setting in which a bridge concept is of the form \(\exists s.X\), where \(X\) is an atomic concept. To construct the corresponding dataset, we focus on subsumptions of the form \(A \sqsubseteq \exists r.B_1\), where \(A\) and \(B_1\) are concept names and \(r\) is a role, together with a pruned ontology \(\mathcal{O}'\) such that \(\mathcal{O} \models A \sqsubseteq \exists r.B_1\) and \(\mathcal{O}' \not\models A \sqsubseteq \exists r.B_1\).
In particular, we seek a bridge concept \(C = \exists r.X\) such that \(\mathcal{O} \models A \sqsubseteq \exists r.X \sqsubseteq \exists r.B_1\).
The target entailment \(A \sqsubseteq \exists r.B_1\) can then potentially be recovered in \(\Omc'\) by restoring the subsumption \(A \sqsubseteq \exists r.X\), the subsumption \(\exists r.X \sqsubseteq \exists r.B_1\), or both, depending on which information has been removed during pruning.

The target subsumptions \(A \sqsubseteq \exists r.B_1\) and their corresponding incomplete ontologies \(\Omc'\) are constructed in three steps, as described below.
\begin{enumerate}[leftmargin=*]
    \item \textbf{Collect and remove.}
    We first collect all existential conjuncts \(\exists r.X\) occurring in axioms of \(\mathcal{O}\) of the form \(A \sqsubseteq \exists r.X \sqcap C\) or \(A \equiv \exists r.X \sqcap C\), where \(A\) and \(X\) are atomic concepts, \(r\) is a role name, and \(C\) is a possibly empty conjunction.

    We then randomly select a fraction \(\rho\) of these existential conjuncts for removal. For each selected conjunct, the corresponding axiom is replaced by \(A \sqsubseteq C\), or removed entirely if \(C\) is empty. In addition, for each removed conjunct \(\exists r.X\), we remove the atomic superclass axioms of its filler \(X\): specifically, an axiom of the form \(X \sqsubseteq B' \sqcap C\), where \(B'\) is atomic, is replaced by \(X \sqsubseteq C\). We denote the resulting ontology by \(\mathcal{O}'\). Since these modifications only remove information from \(\mathcal{O}\), it follows that
    \(
    \mathcal{O} \models \mathcal{O}'.
    \)
    
    \item \textbf{Candidate queries.}
    For each concept \(A\) from which an existential conjunct \(\exists r.X\) has been removed, we generate candidate queries of the form \(A \sqsubseteq \exists r.B_1\), where \(B_1\) is an ancestor of \(X\) whose shortest upward distance from \(X\) in the taxonomy of \(\mathcal{O}\) lies between \(1\) and \(3\). By construction, all such candidate queries are entailed by \(\mathcal{O}\). To avoid overly general choices of \(B_1\), we impose additional constraints on its distance from a root of the taxonomy, its number of subconcepts, and the proportion of fillers of \(r\) that are subsumed by it.

    \item \textbf{Filtering.}
    A candidate query is retained as a positive sample only if pruning has eliminated the corresponding entailment, \textit{i.e.}, \(\mathcal{O}' \not\models A \sqsubseteq \exists r.B_1\).
\end{enumerate}

\textbf{Sampling and negative samples.}
We construct the validation and test sets by sampling positive samples with disjoint sets of left-hand-side concepts \(A\), while also limiting the number of samples associated with each restriction \(\exists r.B_1\). Each positive sample \(A \sqsubseteq \exists r.B_1\) is paired with a negative sample \(A \sqsubseteq \exists r'.B_1^-\), with the same left-hand side \(A\), such that \(\mathcal{O} \not\models A \sqsubseteq \exists r'.B_1^-\).
The restriction \(\exists r'.B_1^-\) must also occur as the right-hand side of a candidate query and satisfy the same structural filtering criteria. We consider two strategies for constructing negative samples:
\begin{itemize}[leftmargin=*]
    \item \textbf{Random:} The restriction \(\exists r'.B_1^-\) is sampled uniformly from the right-hand sides of all candidate queries.
    \item \textbf{Hard (near miss):} We set \(r'=r\) and choose \(B_1^-\) to be close to \(B_1\) in the taxonomy of \(\mathcal{O}\), while having a similar number of subconcepts.
\end{itemize}

\begin{rem}\label{rem:neg}
All positive samples generated above are correct, as they are entailed by the original ontology $\mathcal{O}$. In contrast, under OWL's open-world semantics, non-entailed subsumptions are not necessarily false. To mitigate potential false negatives, we employ an additional semantic audit using an advanced LLM (Claude Opus 5.5). We first screen 5\% of negative samples from each dataset and perform full filtering if more than 5\% of the screened samples of a negative type are flagged as valid or uncertain. In practice, this occurs only for hard negatives, while random negatives generally remain reliable. Although LLM-based filtering cannot guarantee correctness, evaluating both the original and filtered datasets helps assess the robustness of our findings. Further details are provided in Appendix~\ref{app:data}.
\end{rem}

\section{Evaluation}
\label{sec:results}
\begin{table*}[t]
\centering
\caption{Overall results.
A superscript $*$ indicates fine-tuned (FT); best value per column and part in bold.
}
\label{tab:all_results}
\small
\resizebox{\textwidth}{!}{\begin{tabular}{@{}ccccccccccccccccc@{}}
\toprule
& & \multicolumn{9}{c}{\textbf{Atomic setting}} & \multicolumn{6}{c}{\textbf{Existential setting}} \\
\cmidrule(lr){3-11} \cmidrule(lr){12-17}
& & \multicolumn{3}{c}{FoodOn$_A$} & \multicolumn{3}{c}{Snomed$_A$} & \multicolumn{3}{c}{GO$^{+}_{A}$} & \multicolumn{3}{c}{Snomed$_\exists$} & \multicolumn{3}{c}{GO$^{+}_\exists$} \\
\cmidrule(lr){3-5} \cmidrule(lr){6-8} \cmidrule(lr){9-11} \cmidrule(lr){12-14} \cmidrule(lr){15-17}
\textbf{Method} & FT & F1 & X-F1 & X-F1$^*$ & F1 & X-F1 & X-F1$^*$ & F1 & X-F1 & X-F1$^*$ & F1 & X-F1 & X-F1$^*$ & F1 & X-F1 & X-F1$^*$ \\
\midrule
\rowcolor{tabNeg}\multicolumn{17}{c}{\textit{Random Neg}} \\
OnT$^*$ & \checkmark & 0.925 & \NA{} & \NA{} & 0.942 & \NA{} & \NA{} & \textbf{0.963} & \NA{} & \NA{} & 0.800 & \NA{} & \NA{} & \textbf{0.853} & \NA{} & \NA{} \\
SBERT & & 0.721 & \NA{} & \NA{} & 0.930 & \NA{} & \NA{} & 0.721 & \NA{} & \NA{} & 0.745 & \NA{} & \NA{} & 0.673 & \NA{} & \NA{} \\
LLM (Stage~3) & & 0.903 & 0.656 & 0.419 & 0.887 & 0.735 & 0.458 & 0.706 & 0.346 & 0.171 & 0.846 & 0.211 & 0.007 & 0.635 & 0.116 & 0.015 \\
LLM-direct & & 0.835 & 0.814 & 0.366 & 0.773 & 0.679 & 0.343 & 0.623 & 0.579 & 0.200 & 0.788 & 0.188 & 0.008 & 0.503 & 0.075 & 0.000 \\
\cmidrule(l){1-17}
$\ours_{\text{OnT}}^*$ & \checkmark & \textbf{0.960} & \textbf{0.894} & \textbf{0.793} & 0.968 & \textbf{0.849} & \textbf{0.662} & 0.834 & \textbf{0.699} & 0.539 & 0.867 & \textbf{0.528} & \textbf{0.184} & 0.756 & \textbf{0.272} & \textbf{0.156} \\
$\ours_{\text{SBERT}}$ & & \textbf{0.960} & 0.885 & 0.788 & \textbf{0.970} & 0.848 & 0.649 & 0.827 & 0.677 & \textbf{0.562} & \textbf{0.870} & 0.373 & 0.150 & 0.737 & 0.252 & 0.144 \\
\midrule
\rowcolor{tabNeg}\multicolumn{17}{c}{\textit{Hard Neg}} \\
OnT$^*$ & \checkmark & 0.680 & \NA{} & \NA{} & 0.726 & \NA{} & \NA{} & 0.727 & \NA{} & \NA{} & 0.636 & \NA{} & \NA{} & 0.655 & \NA{} & \NA{} \\
SBERT & & 0.634 & \NA{} & \NA{} & 0.665 & \NA{} & \NA{} & 0.650 & \NA{} & \NA{} & 0.641 & \NA{} & \NA{} & 0.667 & \NA{} & \NA{} \\
LLM (Stage~3) & & 0.823 & 0.586 & 0.367 & 0.814 & 0.667 & 0.408 & 0.668 & 0.323 & 0.158 & 0.773 & 0.184 & 0.006 & 0.581 & 0.102 & 0.014 \\
LLM-direct & & 0.788 & 0.767 & 0.337 & 0.733 & 0.641 & 0.324 & 0.604 & 0.561 & 0.192 & 0.765 & 0.180 & 0.008 & 0.477 & 0.070 & 0.000 \\
\cmidrule(l){1-17}
$\ours_{\text{OnT}}^*$ & \checkmark & \textbf{0.845} & \textbf{0.781} & \textbf{0.683} & \textbf{0.862} & \textbf{0.747} & \textbf{0.571} & \textbf{0.766} & \textbf{0.636} & 0.486 & 0.778 & \textbf{0.460} & \textbf{0.156} & \textbf{0.672} & \textbf{0.232} & \textbf{0.131} \\
$\ours_{\text{SBERT}}$ & & 0.844 & 0.772 & 0.679 & 0.856 & 0.738 & 0.553 & 0.759 & 0.614 & \textbf{0.505} & \textbf{0.797} & 0.329 & 0.130 & 0.662 & 0.218 & 0.123 \\
\midrule
\rowcolor{tabNeg}\multicolumn{17}{c}{\textit{Hard Neg (filtered)}} \\
OnT$^*$ & \checkmark & 0.693 & \NA{} & \NA{} & 0.747 & \NA{} & \NA{} & 0.749 & \NA{} & \NA{} & 0.656 & \NA{} & \NA{} & 0.692 & \NA{} & \NA{} \\
SBERT &  & 0.646 & \NA{} & \NA{} & 0.683 & \NA{} & \NA{} & 0.668 & \NA{} & \NA{} & 0.666 & \NA{} & \NA{} & 0.705 & \NA{} & \NA{} \\
LLM (Stage~3) &  & 0.840 & 0.601 & 0.378 & 0.839 & 0.690 & 0.426 & 0.680 & 0.330 & 0.162 & 0.798 & 0.193 & 0.007 & 0.604 & 0.108 & 0.014 \\
LLM-direct &  & 0.799 & 0.777 & 0.344 & 0.752 & 0.658 & 0.335 & 0.612 & 0.569 & 0.195 & 0.781 & 0.186 & 0.008 & 0.492 & 0.073 & 0.000 \\
\cmidrule(l){1-17}
$\ours_{\text{OnT}}^*$ & \checkmark & \textbf{0.865} & \textbf{0.800} & \textbf{0.702} & \textbf{0.889} & \textbf{0.773} & \textbf{0.594} & \textbf{0.785} & \textbf{0.654} & 0.502 & 0.808 & \textbf{0.483} & \textbf{0.165} & \textbf{0.706} & \textbf{0.248} & \textbf{0.141} \\
$\ours_{\text{SBERT}}$ &  & 0.863 & 0.791 & 0.698 & 0.884 & 0.764 & 0.576 & 0.777 & 0.630 & \textbf{0.520} & \textbf{0.827} & 0.347 & 0.138 & 0.696 & 0.233 & 0.132 \\
\bottomrule
\end{tabular}}
\end{table*}

\subsection{Experimental Setup}
\label{sec:setup}

\paragraph{\textbf{Datasets}}
We use three real-world ontologies for our evaluation: FoodOn (2022-08-12 version) \cite{dooley2018foodon}, {\color{black}  Gene Ontology (GO$^{+}$; 2024-11-03 version) \cite{ashburner2000gene}}, and Snomed CT (2023-12-24 version) \cite{SNOMED}.
The datasets are obtained by applying the data construction procedure to these ontologies under two settings:
\begin{enumerate}[leftmargin=*]\sloppy
    \item In the atomic setting{,} where atomic concepts serve as bridge concept candidates, we obtain three atomic datasets, \textbf{FoodOn$_A$},  \textbf{Snomed$_A$}, and \textbf{GO$^{+}_{A}$}. 
    \item In the existential setting, where existential restrictions of the form $\exists s.X$ serve as candidates, the procedure is applied to Snomed CT and GO$^{+}$, producing the existential dataset{s} \textbf{Snomed$_{\exists}$} and \textbf{GO$^{+}_{\exists}$}. We do not apply this procedure to FoodOn as it has too few subsumptions in which existential restrictions serve as bridges.
\end{enumerate}
Each dataset also includes a validation set of 50 positive and 50 negative samples, with the negatives comprising 25 random and 25 hard negatives. This validation set is disjoint from the test set and is used for all tuning, including baseline threshold selection and the Stage~3a prompt.
Statistics of the generated datasets are reported in Table~\ref{tab:ontologies} in the appendix. As mentioned in Remark~\ref{rem:neg}, hard negatives are additionally evaluated on a filtered dataset that excludes samples flagged by our semantic audit.
Note that, in the existential setting, we omit Stage~2b because it generates only atomic concept candidates, while all other stages are retained. For example, given an input subsumption \(A \sqsubseteq \exists r.B_1\), Stage~2a may generate candidates of the form \(\exists r.B'\), where \(B'\) is a child of \(B_1\) in \(\Omc'\).

\paragraph{\textbf{Implementation}}
We test two versions of \ours with  pre-trained SBERT~\cite{reimers2019sentence}
and  OnT~\cite{OnT}
as the embedding model, both using the \texttt{all-MiniLM-L12-v2} variant. While SBERT is used directly, OnT is fine-tuned
on the provided incomplete ontology $\mathcal{O}'$ for 20 epochs, with a learning rate of $10^{-5}$, a training batch size of 256, a clustering loss margin of 3.0, and a centripetal loss margin of 0.5. 
We use the open-source LLM \texttt{Qwen3.5-9B} in non-thinking mode, as it delivers both high performance and fast inference speed. 
In each stage, the LLM verifies the top-$k$ candidates, with $k = 10$ by default on the atomic datasets and $k = 100$ on the existential datasets, whose retrieval space is much larger; there, Stage~2a verifies all direct children of $B_1$, or 100 of them sampled with a fixed seed if $B_1$ has more.

To reduce the search space without compromising correctness, we restrict the candidates in Stage~2 to direct parents or children. Specifically, for a given subsumption $A \sqsubseteq B$, Stage~2a considers only the direct children of $B$\footnote{$A$ is a direct child of $B$ in $\mathcal{O}'$ if $\mathcal{O}' \models A \sqsubseteq B$ and there is no $A'$ distinct from both $A$ and $B$ such that $\mathcal{O}' \models A \sqsubseteq A' \sqsubseteq B$. Direct parents are defined analogously.}, whereas Stage~2b considers only the direct parents of $A$. This restriction substantially reduces the search space and improves the accuracy of the returned bridge concepts. Results obtained using all parents and children of $A$ and $B$ are reported in Appendix~\ref{app:direct_all}. The same procedure is applied in the existential setting for subsumptions of the form $A \sqsubseteq \exists r.B_1$.

\paragraph{\textbf{Baselines}}
To the best of our knowledge, no existing methods directly fit our setting. We therefore introduce four simple baselines that assess the plausibility of given subsumptions{, two of which} do not generate explanations.
The baselines are: (1) two embedding-based methods, \textbf{SBERT} and \textbf{OnT}, which classify subsumptions as true or false by assessing whether their similarity (for SBERT) or subsumption (for OnT) score exceeds a threshold. 
The threshold is selected to maximize the F1 score (defined below) on the validation set of each dataset, containing 50 positive and 50 negative samples. For a fair comparison, we use the same LLM and the same SBERT and OnT as in \ours.
(2) Two LLM-based methods: \textbf{LLM (Stage~3)} and \textbf{LLM-direct}.
\textit{LLM (Stage~3)} corresponds to the standalone Stage~3 of \ours and can be regarded as a retrieval-augmented LLM baseline, using SBERT embeddings to retrieve candidate bridge concepts.
\textit{LLM-direct} extends Stage~3b of \ours by an extra step that prompts the LLM to generate bridge concepts for each accepted subsumption. Specifically, the LLM is provided with the direct parents of $A$ and direct children of $B$ (or $B_1$ in the existential setting), limited to 50 each through random sampling when necessary. In the existential setting, the LLM is additionally given all available roles and instructed to generate bridge concepts of the form $\exists s.X$.

As discussed in the introduction and related work, existing abductive reasoning approaches are not included in our comparisons, as they differ fundamentally from our setting: they focus on the logical procedure of extracting possible missing axioms from a predefined hypothesis space and treat the given subsumption as an actual observation.

\paragraph{\textbf{Evaluation Metrics}} We use the following three metrics:
\begin{itemize}[leftmargin=*]
    \item \textbf{F1}: The standard F1 score that measures the correctness of the \textsc{True}/\textsc{False} prediction, independent of explanations.
    
    \item \textbf{X-F1}: Requires both correct prediction and correct explanation. Specifically, for each positive sample: (1) if a ground-truth bridge concept exists (Stage~2a/2b/3a), the prediction is correct only if the model outputs \textsc{True} and returns one of the ground-truth concepts;
     (2) otherwise, the prediction is correct only if the model outputs \textsc{True} at Stage~3b. For each negative sample, the prediction is correct if the model outputs \textsc{False}. 
    \item \textbf{X-F1$^*$}: The same as X-F1 but  further requires that \textbf{all} output bridge concepts are correct for each positive sample.
\end{itemize}
The embedding-based baseline methods are evaluated only using F1, as they cannot predict missing axioms or provide explanations.
The LLM (Stage~3) and LLM-direct baselines can also be evaluated using X-F1 and X-F1$^*$, since they are capable of generating bridge concepts that serve as explanations.

\subsection{Main Results}

\subsubsection{Summary}
The overall results are presented in Table~\ref{tab:all_results}. Both variants of \ours achieve the best overall performance in terms of F1 score, with particularly consistent improvements in the hard-negative setting. For example, on Snomed$_A$, $\ours_{\text{OnT}}^*$ achieves an F1 score of 0.862, compared with 0.814 for the strongest baseline, LLM (Stage~3). The only exceptions occur in the random-negative setting, where the fine-tuned embedding model OnT outperforms \ours on GO$^{+}_{A}$ and GO$^{+}_\exists$. This may reflect the greater semantic complexity of this ontology (\textit{e.g.}, chemical terminology), where lexical similarities may benefit embedding-based methods while introducing ambiguity for LLMs. Nevertheless, $\ours_{\text{OnT}}^*$ achieves the highest F1 score on GO$^{+}_\exists$ under the more challenging hard-negative setting, both before and after filtering potentially mislabeled negatives, further supporting its effectiveness.

Moreover, \ours consistently outperforms LLM-direct in both X-F1 and X-F1$^*$ across random and hard negatives, with or without filtering. The gains are particularly pronounced in the existential setting and under the stricter X-F1$^*$ metric. Specifically, averaged over all three cases in Table~\ref{tab:all_results}, the X-F1 gap is larger for existential than atomic subsumptions (0.203 vs.\ 0.081), while the gap under X-F1$^*$ exceeds that under X-F1 (0.250 vs.\ 0.130). Notably, in the existential setting, LLM-direct achieves an average X-F1$^*$ of only 0.004, compared with 0.146 for \ours. Since LLM-direct has access to comparable structural information (direct parents and children), these results highlight the benefits of explicitly exploiting ontology structure and embedding-based retrieval over direct LLM reasoning. 
A similar pattern emerges when comparing $\ours_{\text{SBERT}}$ with LLM (Stage~3): $\ours_{\text{SBERT}}$ outperforms this baseline across all datasets and metrics, with larger gains in X-F1 and X-F1$^*$ than in F1 (\textit{e.g.}, 0.788 vs.\ 0.419 in X-F1$^*$ on FoodOn$_A$ with random negatives). 

Nevertheless, explanation quality remains limited across all methods in the existential setting. For \ours, X-F1 and X-F1$^*$ peak at 0.528 and 0.184, respectively, on Snomed$_\exists$ with random negatives, and decrease to 0.272 and 0.156 on GO$^{+}_\exists$. This limitation is largely attributable to the low retrieval performance of Stage~3a in this setting (H@100 $\leq$ 0.212 on GO$^{+}_\exists$; Section~\ref{sec:retrieval}).

\begin{table}[t]
\centering
\caption{Retrieval results (H@k: the proportion of ground-truth candidates ranked within the top $k$). 
}
\label{tab:part1_detailed}
\small
\begin{tabular}{@{}ccccccc@{}}
\toprule
\textbf{Variant} & \textbf{Stage} & \textbf{H@1} & \textbf{H@5} & \textbf{H@10} & \textbf{H@30} & \textbf{H@100} \\
\midrule
\rowcolor{tabData}\multicolumn{7}{c}{\textit{Snomed$_A$}} \\
\multirow{3}{*}{$\ours_{\text{OnT}}^*$} & 2a & 0.664 & 0.950 & 0.974 & 0.995 & 1.000 \\
 & 2b & 0.886 & 1.000 & 1.000 & 1.000 & 1.000 \\
 & 3a & 0.182 & 0.409 & 0.591 & 0.727 & 0.727 \\
\cmidrule(l){1-7}
\multirow{3}{*}{$\ours_{\text{SBERT}}$} & 2a & 0.531 & 0.887 & 0.951 & 0.982 & 0.997 \\
 & 2b & 0.903 & 1.000 & 1.000 & 1.000 & 1.000 \\
 & 3a & 0.318 & 0.591 & 0.682 & 0.818 & 0.955 \\
\midrule
\rowcolor{tabData}\multicolumn{7}{c}{\textit{GO$^{+}_\exists$}} \\
\multirow{2}{*}{$\ours_{\text{OnT}}^*$} & 2a & 0.563 & 0.878 & 0.954 & 0.996 & 1.000 \\
 & 3a & 0.028 & 0.068 & 0.087 & 0.134 & 0.212 \\
\cmidrule(l){1-7}
\multirow{2}{*}{$\ours_{\text{SBERT}}$} & 2a & 0.494 & 0.802 & 0.875 & 0.977 & 0.992 \\
 & 3a & 0.015 & 0.045 & 0.064 & 0.113 & 0.176 \\
\bottomrule
\end{tabular}
\end{table}

We also find that fine-tuning on the target ontology leads to improvements in most cases.
The performance gap is relatively small on atomic datasets such as FoodOn$_A$ and Snomed$_A$, where only atomic concepts are considered as candidates. In contrast, for existential datasets like Snomed$_\exists$, which involve existential restrictions of the form $\exists s.X$, $\ours_{\text{OnT}}^*$ shows substantially better performance in X-F1 (0.528 vs.\ 0.373 with random negatives) due to its adaptation to the underlying logical structure through \mbox{fine-tuning}; on GO$^{+}_\exists$, the gains are consistent but smaller (0.272 vs.\ 0.252). 

\begin{figure}[t]
    \centering
    \input{Figures/fig_per_stage}
 \caption{Detailed results across different stages (plain bars: random negatives; hatched bars: hard negatives; OnT$^*$ and SBERT: $\ours_{\text{OnT}}^*$ and $\ours_{\text{SBERT}}$).}
    \label{fig:stage_combined}
\end{figure}

\begin{table}[t]
\centering
\caption{Stage contribution ablation of $\ours_{\text{OnT}}^*$.
``$\checkmark$'' indicates that a stage is applied.
}
\label{tab:stage_ablation}
\small
\begin{tabular}{@{}cccc ccc ccc@{}}
\toprule
\multicolumn{4}{c}{\textbf{Stages}} & \multicolumn{3}{c}{\textbf{Random Neg}} & \multicolumn{3}{c}{\textbf{Hard Neg}} \\
\cmidrule(lr){1-4} \cmidrule(lr){5-7} \cmidrule(lr){8-10}
2a & 2b & 3a & 3b & F1 & X-F1 & X-F1$^*$ & F1 & X-F1 & X-F1$^*$ \\
\midrule
\rowcolor{tabData}\multicolumn{10}{c}{\textit{Snomed$_A$}} \\
$\checkmark$ & -- & -- & -- & 0.781 & 0.692 & 0.481 & 0.742 & 0.655 & 0.451 \\
-- & $\checkmark$ & -- & -- & 0.718 & 0.667 & 0.611 & 0.673 & 0.622 & 0.569 \\
-- & -- & $\checkmark$ & -- & 0.672 & 0.556 & 0.356 & 0.634 & 0.524 & 0.335 \\
-- & -- & -- & $\checkmark$ & 0.773 & \NA{} & \NA{} & 0.733 & \NA{} & \NA{} \\
\cmidrule(l){1-10}
$\checkmark$ & $\checkmark$ & $\checkmark$ & $\checkmark$ & \textbf{0.968} & \textbf{0.849} & \textbf{0.662} & \textbf{0.862} & \textbf{0.747} & \textbf{0.571} \\
\midrule
\rowcolor{tabData}\multicolumn{10}{c}{\textit{GO$^{+}_\exists$}} \\
$\checkmark$ & -- & -- & -- & 0.523 & 0.199 & 0.139 & 0.483 & 0.181 & 0.126 \\
-- & -- & $\checkmark$ & -- & 0.499 & 0.161 & 0.031 & 0.453 & 0.143 & 0.027 \\
-- & -- & -- & $\checkmark$ & 0.503 & \NA{} & \NA{} & 0.477 & \NA{} & \NA{} \\
\cmidrule(l){1-10}
$\checkmark$ & -- & $\checkmark$ & $\checkmark$ & \textbf{0.756} & \textbf{0.272} & \textbf{0.156} & \textbf{0.672} & \textbf{0.232} & \textbf{0.131} \\
\bottomrule
\end{tabular}
\end{table}

\subsection{Detailed Analysis}

\subsubsection{Retrieval Performance}\label{sec:retrieval}
The retrieval results are shown in Table~\ref{tab:part1_detailed}. Due to space limitations, only part of the results is provided, and other results can be found in Appendix~\ref{app:go_results} (Table~\ref{tab:part1_go}). 
We categorize the positive samples according to whether they admit downward bridge concepts in Stage~2a, upward bridge concepts in Stage~2b (these two groups may overlap), or only bidirectional bridge concepts in Stage~3a (disjoint from the former groups). The existential setting uses the same categorization, except that Stage~2b is not applicable.

We observe that the retrieval performance of Stage~2a/2b is substantially higher than that of Stage~3a alone. This is reasonable, as the number of candidates in Stage~2a/2b (\textit{i.e.}, direct children/parents of $B/A$) is significantly smaller than the space of all atomic concepts. Furthermore, Stage~2b achieves better performance than Stage~2a due to the smaller number of direct parents compared to direct children (\textit{e.g.}, 1.25 vs.\ 32.7 on average in Snomed$_A$).

Moreover, because of the substantially larger candidate space of existential restrictions of the form $\exists s.X$, the retrieval performance of Stage~3a on the GO$^{+}_\exists$ and Snomed$_\exists$ datasets is considerably lower than on the atomic datasets. By fine-tuning on the given incomplete ontologies, OnT consistently outperforms SBERT in this stage across all cutoffs (H@100: 0.212 vs.\ 0.176 on GO$^{+}_\exists$ and 0.583 vs.\ 0.209 on Snomed$_\exists$). Nevertheless, a substantial gap remains relative to the atomic setting, in which only atomic concepts are considered as bridge concepts.

\subsubsection{Per-Stage Performance}
Figure~\ref{fig:stage_combined} presents a detailed breakdown of the results at each stage, including true positives (TP), false positives (FP), false negatives (FN), and true negatives (TN).  

For Snomed$_A$, the majority of TP predictions are generated in Stage~2a, followed by Stages~2b and~3a. In contrast, most FPs under F1 originate from Stage~3a. This is consistent with the retrieval results in Table~\ref{tab:part1_detailed}, where Stage~3a performs worse than Stages~2a and~2b. However, Stage~2a also accepts a considerable number of incorrect bridge concepts, which leads to a large proportion of FPs under X-F1 and X-F1$^*$. 
This pattern is observed consistently across both datasets in Figure~\ref{fig:stage_combined} and under both random and hard negatives.

Moreover, we observe substantially more FNs in the existential setting (here, GO$^{+}_\exists$) than in the atomic setting (here, Snomed$_A$). This is mainly because many more positive samples reach Stage~3b (445 vs.\ 84 for $\ours_{\text{OnT}}^*$), and the LLM rejects a larger proportion of them (79\% vs.\ 69\%). This suggests that the LLM is more prone to incorrectly rejecting subsumptions of the form $A \sqsubseteq \exists r.B_1$ than simpler subsumptions of the form $A \sqsubseteq B$.

\newif\ifcasefigure
\casefiguretrue
\subsubsection{Ablation Study}
Table~\ref{tab:stage_ablation} compares individual stages with the full pipeline on Snomed$_A$ and GO$^{+}_\exists$. Since Stage~3b directly verifies subsumptions without generating bridge concepts, X-F1 and X-F1$^*$ are not applicable.
On Snomed$_A$, the structure-guided Stages~2a/2b outperform Stage~3a alone, particularly in explanation quality. For instance, under random negatives, Stage~3a achieves an X-F1$^*$ of 0.356, compared with 0.481 and 0.611 for Stages~2a and~2b, respectively. Stage~3b alone achieves reasonable F1 scores but remains substantially below the full pipeline (0.773 vs.\ 0.968 on Snomed$_A$ and 0.503 vs.\ 0.756 on GO$^{+}_\exists$).
The full pipeline consistently achieves the best performance across all reported metrics and settings. Notably, on GO$^{+}_\exists$ with random negatives, combining the stages improves X-F1$^*$ from 0.031 (Stage~3a alone) to 0.156. These results highlight the complementary benefits of structure-guided reasoning, embedding-based retrieval, and direct LLM verification.

\ifcasefigure\suppressfloats[t]
\begin{figure}[t]
\centering
\definecolor{csInk}{HTML}{262626}
\definecolor{csMuted}{HTML}{5E6670}
\definecolor{csFill}{HTML}{F3F5F7}
\definecolor{csBlueFill}{HTML}{E8EEF4}
\definecolor{csNode}{HTML}{8C949D}
\definecolor{csBridge}{HTML}{5A7FA6}
\definecolor{csDark}{HTML}{263F5C}
\definecolor{csMiss}{HTML}{B23A32}
\definecolor{csBad}{HTML}{9A4A3F}
\definecolor{csBadFill}{HTML}{F6ECEA}
\definecolor{csGrid}{HTML}{E4E4E4}
\definecolor{csGood}{HTML}{2E7D4F}
\definecolor{csGoodFill}{HTML}{E6F2EA}
\begin{tikzpicture}[x=1pt, y=1pt, font=\small, text=csInk,
  cnode/.style={draw=csNode, line width=0.5pt, rounded corners=2pt, fill=white, inner xsep=3pt, inner ysep=2pt, font=\footnotesize\ttfamily, align=center},
  bnode/.style={cnode, draw=csGood, line width=0.6pt, fill=csGoodFill},
  wnode/.style={cnode, draw=csBad, line width=0.6pt, fill=csBadFill},
  sym/.style={inner sep=0pt, font=\footnotesize},
  ann/.style={inner sep=0pt, font=\scriptsize, text=csMuted},
  sub/.style={-{Stealth[length=3pt]}, csMuted, line width=0.6pt},
  miss/.style={-{Stealth[length=3pt]}, csMiss, line width=0.75pt, dash pattern=on 2.2pt off 1.4pt},
  wrong/.style={-{Stealth[length=3pt]}, csMiss, line width=0.75pt, dash pattern=on 2.2pt off 1.4pt},
  good/.style={-{Stealth[length=3pt]}, csGood, line width=0.75pt, dash pattern=on 2.2pt off 1.4pt},
  elabgood/.style={inner sep=0pt, font=\scriptsize\itshape, text=csGood},
  elab/.style={inner sep=0pt, font=\scriptsize\itshape, text=csMiss},
  elabbad/.style={inner sep=0pt, font=\small\bfseries, text=csMiss},
  itb/.style={circle, draw=csBridge, line width=0.5pt, fill=white, text=csDark, inner sep=0pt, minimum size=8.4pt, font=\scriptsize\bfseries}]
\node[circle, fill=csGoodFill, draw=csGood, line width=0.5pt, text=csGood, inner sep=0pt, minimum size=9.6pt, font=\scriptsize\bfseries] at (9.5,-7.0) {$\checkmark$};
\node[anchor=base west, inner sep=0pt, font=\small\bfseries] at (16.0,-10.0) {True positive};
\node[anchor=base east, inner sep=0pt, font=\footnotesize\itshape, text=csMuted] at (111.5,-10.0) {Snomed$_A$};
\node[circle, fill=csBadFill, draw=csMiss, line width=0.5pt, text=csMiss, inner sep=0pt, minimum size=9.6pt, font=\scriptsize\bfseries] at (135.0,-7.0) {$\times$};
\node[anchor=base west, inner sep=0pt, font=\small\bfseries] at (141.5,-10.0) {False positive};
\node[anchor=base east, inner sep=0pt, font=\footnotesize\itshape, text=csMuted] at (237.0,-10.0) {FoodOn$_A$};
\node[cnode, text width=81.5pt, anchor=north] (a0) at (59.8,-20.0) {Structure of visual system};
\node[sym, left=3pt of a0] {$B$};
\node[bnode, text width=81.5pt, below=11.0pt of a0] (a1) {Structure of valve of lacrimal canaliculus};
\node[sym, left=3pt of a1] {$C$};
\node[cnode, text width=81.5pt, below=11.0pt of a1] (a2) {Structure of valve of Bochdalek};
\node[sym, left=3pt of a2] {$A$};
\draw[good] (a1) -- node[elabgood, right=2pt] {$\checkmark$\,missing} (a0);
\draw[sub] (a2) -- (a1);
\node[cnode, text width=81.5pt] (b0) at (185.2,0 |- a0) {butterfly mackerel};
\node[sym, left=3pt of b0] {$B$};
\node[wnode, text width=81.5pt] (b1) at (185.2,0 |- a1) {frigate mackerel};
\node[sym, left=3pt of b1] {$C$};
\node[cnode, text width=81.5pt] (b2) at (185.2,0 |- a2) {Katsuwonus};
\node[sym, left=3pt of b2] {$A$};
\draw[wrong] (b1) -- node[elabbad, right=2pt] {$\boldsymbol{\times}$} (b0);
\draw[wrong] (b2) -- node[elabbad, right=2pt] {$\boldsymbol{\times}$} (b1);
\path (a2.south) ++(0,-12.0) coordinate (a2b);
\path (16.0,0 |- a2b) ++(0,-0.0) node[anchor=base west, inner sep=0pt] (a2t0) {$\ours_{\text{OnT}}^*$: \textbf{Yes} (Stage~2b)};
\fill[csGood] (9.5,0 |- a2t0.base) ++(0,2.0) circle (1.1pt);
\path (16.0,0 |- a2b) ++(0,-10.4) node[anchor=base west, inner sep=0pt] (a2t1) {LLM-direct: \textbf{No}};
\fill[csGood] (9.5,0 |- a2t1.base) ++(0,2.0) circle (1.1pt);
\path (a2.south) ++(0,-12.0) coordinate (blb);
\path (141.5,0 |- blb) ++(0,-0.0) node[anchor=base west, inner sep=0pt] (blt0) {$\ours_{\text{OnT}}^*$: \textbf{Yes} (Stage~3a)};
\fill[csBad] (135.0,0 |- blt0.base) ++(0,2.0) circle (1.1pt);
\path (141.5,0 |- blb) ++(0,-10.4) node[anchor=base west, inner sep=0pt] (blt1) {but $C$ is a wrong bridge};
\fill[csBad] (135.0,0 |- blt1.base) ++(0,2.0) circle (1.1pt);
\begin{scope}[on background layer]
\path (0,0) coordinate (cardtop); \path (0,0 |- a2t1.base) ++(0,-5) coordinate (cardbot);
\fill[csGoodFill!70, rounded corners=2.5pt] (0,0 |- cardtop) rectangle (115.5,0 |- cardbot);
\fill[csGood] (0,0 |- cardtop) rectangle (1.8,0 |- cardbot);
\path (0,0) coordinate (cardtop); \path (0,0 |- blt1.base) ++(0,-5) coordinate (cardbot);
\fill[csBadFill!70, rounded corners=2.5pt] (125.5,0 |- cardtop) rectangle (241.0,0 |- cardbot);
\fill[csBad] (125.5,0 |- cardtop) rectangle (127.3,0 |- cardbot);
\end{scope}
\end{tikzpicture}
\caption{Single-iteration cases (solid arrows: subsumptions in $\mathcal{O}'$; dashed arrows: returned missing axioms).}
\label{fig:cases}
\end{figure}\fi

\subsection{Case Studies}\label{sec:case_study}
\ifcasefigure
\textbf{Single-iteration case{.}}
Figure~\ref{fig:cases} shows one true positive (TP) and one false positive (FP) case.
In the TP case, LLM-direct rejects $A \sqsubseteq B$. \ours{} instead finds the direct parent $C$ of $A$ in Stage~2b, whose label makes the location of the valve (the lacrimal canaliculus) explicit; the LLM then validates the missing axiom $C \sqsubseteq B$, and $\{C \sqsubseteq B\} \cup J_{A \sqsubseteq C}$ explains the answer.
In the FP case, $A$ and $C$ are distinct fishes of the mackerel family. The candidate $C$ comes from the embedding-based retrieval of Stage~3a, and its close lexical and taxonomic similarity to $A$ and $B$ misleads the LLM.

\else

\textbf{Single-iteration case}  
We present two illustrative cases: one true positive (TP) and one false positive (FP). Assuming the given subsumption is always $A\sqsubseteq B$.

\begin{enumerate}[leftmargin=*]\sloppy
    \item \textbf{(TP)} In Snomed$_A$, in the case
    \begin{align*}
    A ={}& \texttt{Structure of valve of Bochdalek},\\
    B ={}& \texttt{Structure of visual system},
    \end{align*}
    \ours correctly identifies the subsumption $A \sqsubseteq B$ by discovering the bridge concept $C=$ {\ttfamily   Structure of valve of lacrimal canaliculus} in Stage~2b. The LLM-direct baseline incorrectly classifies this subsumption as false, likely due to the large semantic gap between $A$ and $B$: the eponym in $A$ hides that the valve belongs to the lacrimal canaliculus.
    
    \item \textbf{(FP)} In FoodOn$_A$, in the case $A =$ {\ttfamily   Katsuwonus} and $B = $ {\ttfamily   butterfly mackerel}. $\ours_{\text{OnT}}^*$ incorrectly identifies $C = $ {\ttfamily   frigate mackerel} as a valid bridge concept in Stage~3a. This is erroneous, as neither $A \sqsubseteq C$ nor $C \sqsubseteq B$ holds: they are distinct fishes of the mackerel family. 
\end{enumerate}
\fi

\begin{sloppypar}\textbf{Multi-iteration case{.}}
Here, we apply \ours iteratively: the missing subsumption returned in one iteration is the target subsumption of the next one.
On Snomed$_A$, assume the initial input is  $A\sqsubseteq B$ with
$A$ = \texttt{Primary undifferentiated carcinoma of anterior wall of nasopharynx}, $B$ = \texttt{Disorder of nasopharynx},
and \ours performs four iterations:
\end{sloppypar}
\begin{itemize}[leftmargin=*]\sloppy
    \item \textbf{Iteration 1}: for $A\sqsubseteq B$, \ours finds \(C_1 =\) {\ttfamily Lesion of nasopharynx} in Stage~2a and returns \(A \sqsubseteq C_1\).
    \item \textbf{Iteration 2}: for $A\sqsubseteq C_1$, it finds \(C_2 =\) {\ttfamily Neoplasm of anterior wall of nasopharynx} in Stage~2a and returns \(A \sqsubseteq C_2\).
    \item \textbf{Iteration 3}: for \(A \sqsubseteq C_2\), it finds \(C_3 =\) {\ttfamily Malignant tumour of anterior wall of \seqsplit{nasopharynx}} in Stage~2a and returns \(A \sqsubseteq C_3\).
    \item \textbf{Iteration 4}: for \(A \sqsubseteq C_3\), it finds \(C_4 =\) {\ttfamily Primary malignant neoplasm of anterior wall of nasopharynx} in Stage~3a and returns \(A \sqsubseteq C_4\) and \(C_4 \sqsubseteq C_3\).
\end{itemize}
In summary, the four iterations produce the subsumption chain
\(
A \sqsubseteq C_4 \sqsubseteq C_3 \sqsubseteq C_2 \sqsubseteq C_1 \sqsubseteq B,
\)
where all the identified missing axioms are correct and exhibit progressively finer granularity. This example demonstrates the potential of iterative applications of \ours to yield increasingly precise and fine-grained results.

\section{Conclusion}

In this work, we proposed \ours, an end-to-end framework for explainable neural-symbolic reasoning over incomplete OWL ontologies. By integrating the logical structure and textual information of the ontology via embeddings and LLMs, \ours can infer the plausibility of a concept subsumption axiom and{,} in addition{,} can return a set of axioms as its explanation, preferring stated axioms in the ontology when possible{,} with up to two predicted missing axioms.
Comprehensive experiments and case studies on five datasets based on three real-world ontologies have demonstrated the good overall performance of \ours and the effectiveness of its modules.
In future work, we aim to enhance the retrieval procedure in \ours{} to identify complex bridge concepts more efficiently and to broaden the evaluation to more expressive forms beyond the $\exists s.X$ form considered in this work.
We would also like to implement \ours as a reasoner compatible with Protégé.

\bibliographystyle{ACM-Reference-Format}
\bibliography{main}


\begin{thebibliography}{45}


\ifx \showCODEN    \undefined \def \showCODEN     #1{\unskip}     \fi
\ifx \showISBNx    \undefined \def \showISBNx     #1{\unskip}     \fi
\ifx \showISBNxiii \undefined \def \showISBNxiii  #1{\unskip}     \fi
\ifx \showISSN     \undefined \def \showISSN      #1{\unskip}     \fi
\ifx \showLCCN     \undefined \def \showLCCN      #1{\unskip}     \fi
\ifx \shownote     \undefined \def \shownote      #1{#1}          \fi
\ifx \showarticletitle \undefined \def \showarticletitle #1{#1}   \fi
\ifx \showURL      \undefined \def \showURL       {\relax}        \fi
\providecommand\bibfield[2]{#2}
\providecommand\bibinfo[2]{#2}
\providecommand\natexlab[1]{#1}
\providecommand\showeprint[2][]{arXiv:#2}

\bibitem[Alrabbaa et~al\mbox{.}(2024)]%
        {alrabbaa2024explaining}
\bibfield{author}{\bibinfo{person}{Christian Alrabbaa}, \bibinfo{person}{Stefan Borgwardt}, \bibinfo{person}{Tom Friese}, \bibinfo{person}{Anke Hirsch}, \bibinfo{person}{Nina Knieriemen}, \bibinfo{person}{Patrick Koopmann}, \bibinfo{person}{Alisa Kovtunova}, \bibinfo{person}{Antonio Kr{\"{u}}ger}, \bibinfo{person}{Alexej Popovi{\v{c}}}, {and} \bibinfo{person}{Ida Siahaan}.} \bibinfo{year}{2024}\natexlab{}.
\newblock \showarticletitle{Explaining Reasoning Results for {OWL} Ontologies with {Evee}}. In \bibinfo{booktitle}{\emph{Proceedings of the 21st International Conference on Principles of Knowledge Representation and Reasoning (KR 2024)}}. \bibinfo{pages}{709--719}.
\newblock
\href{https://doi.org/10.24963/kr.2024/67}{doi:\nolinkurl{10.24963/kr.2024/67}}


\bibitem[Ashburner et~al\mbox{.}(2000)]%
        {ashburner2000gene}
\bibfield{author}{\bibinfo{person}{Michael Ashburner}, \bibinfo{person}{Catherine~A. Ball}, \bibinfo{person}{Judith~A. Blake}, \bibinfo{person}{David Botstein}, \bibinfo{person}{Heather Butler}, \bibinfo{person}{J.~Michael Cherry}, \bibinfo{person}{Allan~P. Davis}, \bibinfo{person}{Kara Dolinski}, \bibinfo{person}{Selina~S. Dwight}, \bibinfo{person}{Janan~T. Eppig}, {et~al\mbox{.}}} \bibinfo{year}{2000}\natexlab{}.
\newblock \showarticletitle{Gene {O}ntology: tool for the unification of biology}.
\newblock \bibinfo{journal}{\emph{Nature Genetics}} \bibinfo{volume}{25}, \bibinfo{number}{1} (\bibinfo{year}{2000}), \bibinfo{pages}{25--29}.
\newblock


\bibitem[Baader et~al\mbox{.}(2003)]%
        {baader2003description}
\bibfield{editor}{\bibinfo{person}{Franz Baader}, \bibinfo{person}{Diego Calvanese}, \bibinfo{person}{Deborah~L. McGuinness}, \bibinfo{person}{Daniele Nardi}, {and} \bibinfo{person}{Peter~F. Patel-Schneider}} (Eds.). \bibinfo{year}{2003}\natexlab{}.
\newblock \bibinfo{booktitle}{\emph{The Description Logic Handbook: Theory, Implementation, and Applications}}.
\newblock \bibinfo{publisher}{Cambridge University Press}.
\newblock


\bibitem[Baader et~al\mbox{.}(2017)]%
        {DBLP:books/daglib/0041477}
\bibfield{author}{\bibinfo{person}{Franz Baader}, \bibinfo{person}{Ian Horrocks}, \bibinfo{person}{Carsten Lutz}, {and} \bibinfo{person}{Ulrike Sattler}.} \bibinfo{year}{2017}\natexlab{}.
\newblock \bibinfo{booktitle}{\emph{An Introduction to Description Logic}}.
\newblock \bibinfo{publisher}{Cambridge University Press}.
\newblock
\showISBNx{978-0-521-69542-8}


\bibitem[Babaei~Giglou et~al\mbox{.}(2023)]%
        {babaei2023llms4ol}
\bibfield{author}{\bibinfo{person}{Hamed Babaei~Giglou}, \bibinfo{person}{Jennifer D'Souza}, {and} \bibinfo{person}{S{\"o}ren Auer}.} \bibinfo{year}{2023}\natexlab{}.
\newblock \showarticletitle{{LLMs4OL}: Large Language Models for Ontology Learning}. In \bibinfo{booktitle}{\emph{The Semantic Web -- ISWC 2023}}. \bibinfo{publisher}{Springer}, \bibinfo{pages}{408--427}.
\newblock


\bibitem[Bienvenu(2008)]%
        {bienvenu2008complexity}
\bibfield{author}{\bibinfo{person}{Meghyn Bienvenu}.} \bibinfo{year}{2008}\natexlab{}.
\newblock \showarticletitle{Complexity of Abduction in the {$\mathcal{EL}$} Family of Lightweight Description Logics}. In \bibinfo{booktitle}{\emph{Proceedings of the 11th International Conference on Principles of Knowledge Representation and Reasoning (KR 2008)}}. \bibinfo{pages}{220--230}.
\newblock


\bibitem[Chen et~al\mbox{.}(2023)]%
        {chen2023contextual}
\bibfield{author}{\bibinfo{person}{Jiaoyan Chen}, \bibinfo{person}{Yuan He}, \bibinfo{person}{Yuxia Geng}, \bibinfo{person}{Ernesto Jim{\'e}nez-Ruiz}, \bibinfo{person}{Hang Dong}, {and} \bibinfo{person}{Ian Horrocks}.} \bibinfo{year}{2023}\natexlab{}.
\newblock \showarticletitle{Contextual semantic embeddings for ontology subsumption prediction}.
\newblock \bibinfo{journal}{\emph{World Wide Web}} \bibinfo{volume}{26}, \bibinfo{number}{5} (\bibinfo{year}{2023}), \bibinfo{pages}{2569--2591}.
\newblock


\bibitem[Chen et~al\mbox{.}(2021)]%
        {DBLP:journals/ml/ChenHJHAH21}
\bibfield{author}{\bibinfo{person}{Jiaoyan Chen}, \bibinfo{person}{Pan Hu}, \bibinfo{person}{Ernesto Jim{\'{e}}nez{-}Ruiz}, \bibinfo{person}{Ole~Magnus Holter}, \bibinfo{person}{Denvar Antonyrajah}, {and} \bibinfo{person}{Ian Horrocks}.} \bibinfo{year}{2021}\natexlab{}.
\newblock \showarticletitle{OWL2Vec*: embedding of {OWL} ontologies}.
\newblock \bibinfo{journal}{\emph{Machine Learning}} \bibinfo{volume}{110}, \bibinfo{number}{7} (\bibinfo{year}{2021}), \bibinfo{pages}{1813--1845}.
\newblock
\href{https://doi.org/10.1007/s10994-021-05997-6}{doi:\nolinkurl{10.1007/s10994-021-05997-6}}


\bibitem[Chen et~al\mbox{.}(2025)]%
        {chen2024ontology}
\bibfield{author}{\bibinfo{person}{Jiaoyan Chen}, \bibinfo{person}{Olga Mashkova}, \bibinfo{person}{Fernando Zhapa-Camacho}, \bibinfo{person}{Robert Hoehndorf}, \bibinfo{person}{Yuan He}, {and} \bibinfo{person}{Ian Horrocks}.} \bibinfo{year}{2025}\natexlab{}.
\newblock \showarticletitle{Ontology Embedding: A Survey of Methods, Applications and Resources}.
\newblock \bibinfo{journal}{\emph{IEEE Transactions on Knowledge and Data Engineering}} \bibinfo{volume}{37}, \bibinfo{number}{7} (\bibinfo{year}{2025}), \bibinfo{pages}{4193--4212}.
\newblock


\bibitem[Del-Pinto and Schmidt(2019)]%
        {del2019abox}
\bibfield{author}{\bibinfo{person}{Warren Del-Pinto} {and} \bibinfo{person}{Renate~A. Schmidt}.} \bibinfo{year}{2019}\natexlab{}.
\newblock \showarticletitle{{ABox} Abduction via Forgetting in {$\mathcal{ALC}$}}. In \bibinfo{booktitle}{\emph{Proceedings of the AAAI Conference on Artificial Intelligence}}, Vol.~\bibinfo{volume}{33}. \bibinfo{pages}{2768--2775}.
\newblock


\bibitem[Donnelly(2006)]%
        {SNOMED}
\bibfield{author}{\bibinfo{person}{Kevin Donnelly}.} \bibinfo{year}{2006}\natexlab{}.
\newblock \showarticletitle{{SNOMED-CT}: The advanced terminology and coding system for {eHealth}}.
\newblock \bibinfo{journal}{\emph{Studies in Health Technology and Informatics}}  \bibinfo{volume}{121} (\bibinfo{year}{2006}), \bibinfo{pages}{279--290}.
\newblock


\bibitem[Dooley et~al\mbox{.}(2018)]%
        {dooley2018foodon}
\bibfield{author}{\bibinfo{person}{Damion~M. Dooley}, \bibinfo{person}{Emma~J. Griffiths}, \bibinfo{person}{Gurinder~S. Gosal}, \bibinfo{person}{Pier~L. Buttigieg}, \bibinfo{person}{Robert Hoehndorf}, \bibinfo{person}{Matthew~C. Lange}, \bibinfo{person}{Lynn~M. Schriml}, \bibinfo{person}{Fiona S.~L. Brinkman}, {and} \bibinfo{person}{William W.~L. Hsiao}.} \bibinfo{year}{2018}\natexlab{}.
\newblock \showarticletitle{{FoodOn}: a harmonized food ontology to increase global food traceability, quality control and data integration}.
\newblock \bibinfo{journal}{\emph{npj Science of Food}} \bibinfo{volume}{2}, \bibinfo{number}{1} (\bibinfo{year}{2018}), \bibinfo{pages}{23}.
\newblock


\bibitem[Du et~al\mbox{.}(2012)]%
        {du2012towards}
\bibfield{author}{\bibinfo{person}{Jianfeng Du}, \bibinfo{person}{Guilin Qi}, \bibinfo{person}{Yi-Dong Shen}, {and} \bibinfo{person}{Jeff~Z. Pan}.} \bibinfo{year}{2012}\natexlab{}.
\newblock \showarticletitle{Towards practical {AB}ox abduction in large description logic ontologies}.
\newblock \bibinfo{journal}{\emph{International Journal on Semantic Web and Information Systems (IJSWIS)}} \bibinfo{volume}{8}, \bibinfo{number}{2} (\bibinfo{year}{2012}), \bibinfo{pages}{1--33}.
\newblock


\bibitem[Du et~al\mbox{.}(2017)]%
        {du2017practical}
\bibfield{author}{\bibinfo{person}{Jianfeng Du}, \bibinfo{person}{Hai Wan}, {and} \bibinfo{person}{Huaguan Ma}.} \bibinfo{year}{2017}\natexlab{}.
\newblock \showarticletitle{Practical {TBox} Abduction Based on Justification Patterns}. In \bibinfo{booktitle}{\emph{Proceedings of the AAAI Conference on Artificial Intelligence}}, Vol.~\bibinfo{volume}{31}. \bibinfo{pages}{1100--1106}.
\newblock


\bibitem[Du et~al\mbox{.}(2014)]%
        {du2014tractable}
\bibfield{author}{\bibinfo{person}{Jianfeng Du}, \bibinfo{person}{Kewen Wang}, {and} \bibinfo{person}{Yi-Dong Shen}.} \bibinfo{year}{2014}\natexlab{}.
\newblock \showarticletitle{A Tractable Approach to {ABox} Abduction over Description Logic Ontologies}. In \bibinfo{booktitle}{\emph{Proceedings of the AAAI Conference on Artificial Intelligence}}, Vol.~\bibinfo{volume}{28}. \bibinfo{pages}{1034--1040}.
\newblock


\bibitem[Elsenbroich et~al\mbox{.}(2006)]%
        {DBLP:conf/owled/ElsenbroichKS06}
\bibfield{author}{\bibinfo{person}{Corinna Elsenbroich}, \bibinfo{person}{Oliver Kutz}, {and} \bibinfo{person}{Ulrike Sattler}.} \bibinfo{year}{2006}\natexlab{}.
\newblock \showarticletitle{A Case for Abductive Reasoning over Ontologies}. In \bibinfo{booktitle}{\emph{Proceedings of the OWLED*06 Workshop on {OWL:} Experiences and Directions}} \emph{(\bibinfo{series}{{CEUR} Workshop Proceedings}, Vol.~\bibinfo{volume}{216})}. \bibinfo{publisher}{CEUR-WS.org}.
\newblock


\bibitem[Glimm et~al\mbox{.}(2014)]%
        {glimm2014hermit}
\bibfield{author}{\bibinfo{person}{Birte Glimm}, \bibinfo{person}{Ian Horrocks}, \bibinfo{person}{Boris Motik}, \bibinfo{person}{Giorgos Stoilos}, {and} \bibinfo{person}{Zhe Wang}.} \bibinfo{year}{2014}\natexlab{}.
\newblock \showarticletitle{{HermiT}: An {OWL} 2 Reasoner}.
\newblock \bibinfo{journal}{\emph{Journal of Automated Reasoning}} \bibinfo{volume}{53}, \bibinfo{number}{3} (\bibinfo{year}{2014}), \bibinfo{pages}{245--269}.
\newblock


\bibitem[Haak et~al\mbox{.}(2025)]%
        {haak2025not}
\bibfield{author}{\bibinfo{person}{Anselm Haak}, \bibinfo{person}{Patrick Koopmann}, \bibinfo{person}{Yasir Mahmood}, {and} \bibinfo{person}{Anni-Yasmin Turhan}.} \bibinfo{year}{2025}\natexlab{}.
\newblock \showarticletitle{Why not? {D}eveloping {ABox} Abduction Beyond Repairs}. In \bibinfo{booktitle}{\emph{Proceedings of the 38th International Workshop on Description Logics (DL 2025)}} \emph{(\bibinfo{series}{{CEUR} Workshop Proceedings}, Vol.~\bibinfo{volume}{4091})}.
\newblock


\bibitem[Haifani et~al\mbox{.}(2022)]%
        {haifani2022connection}
\bibfield{author}{\bibinfo{person}{Fajar Haifani}, \bibinfo{person}{Patrick Koopmann}, \bibinfo{person}{Sophie Tourret}, {and} \bibinfo{person}{Christoph Weidenbach}.} \bibinfo{year}{2022}\natexlab{}.
\newblock \showarticletitle{Connection-Minimal Abduction in {$\mathcal{EL}$} via Translation to {FOL}}. In \bibinfo{booktitle}{\emph{International Joint Conference on Automated Reasoning (IJCAR 2022)}}. \bibinfo{publisher}{Springer}, \bibinfo{pages}{188--207}.
\newblock


\bibitem[Halland and Britz(2012)]%
        {halland2012abox}
\bibfield{author}{\bibinfo{person}{Ken Halland} {and} \bibinfo{person}{Katarina Britz}.} \bibinfo{year}{2012}\natexlab{}.
\newblock \showarticletitle{{ABox} abduction in {$\mathcal{ALC}$} using a {DL} tableau}. In \bibinfo{booktitle}{\emph{Proceedings of the South African Institute for Computer Scientists and Information Technologists Conference}}. \bibinfo{pages}{51--58}.
\newblock


\bibitem[He et~al\mbox{.}(2023)]%
        {ontolama}
\bibfield{author}{\bibinfo{person}{Yuan He}, \bibinfo{person}{Jiaoyan Chen}, \bibinfo{person}{Ernesto Jim{\'{e}}nez{-}Ruiz}, \bibinfo{person}{Hang Dong}, {and} \bibinfo{person}{Ian Horrocks}.} \bibinfo{year}{2023}\natexlab{}.
\newblock \showarticletitle{Language Model Analysis for Ontology Subsumption Inference}. In \bibinfo{booktitle}{\emph{Findings of the Association for Computational Linguistics: {ACL} 2023}}. \bibinfo{pages}{3439--3453}.
\newblock
\href{https://doi.org/10.18653/v1/2023.findings-acl.213}{doi:\nolinkurl{10.18653/v1/2023.findings-acl.213}}


\bibitem[He et~al\mbox{.}(2024)]%
        {hit}
\bibfield{author}{\bibinfo{person}{Yuan He}, \bibinfo{person}{Zhangdie Yuan}, \bibinfo{person}{Jiaoyan Chen}, {and} \bibinfo{person}{Ian Horrocks}.} \bibinfo{year}{2024}\natexlab{}.
\newblock \showarticletitle{Language Models as Hierarchy Encoders}. In \bibinfo{booktitle}{\emph{Advances in Neural Information Processing Systems 37 (NeurIPS 2024)}}. \bibinfo{pages}{14690--14711}.
\newblock
\href{https://doi.org/10.52202/079017-0469}{doi:\nolinkurl{10.52202/079017-0469}}


\bibitem[Jackermeier et~al\mbox{.}(2024)]%
        {DBLP:conf/www/JackermeierC024}
\bibfield{author}{\bibinfo{person}{Mathias Jackermeier}, \bibinfo{person}{Jiaoyan Chen}, {and} \bibinfo{person}{Ian Horrocks}.} \bibinfo{year}{2024}\natexlab{}.
\newblock \showarticletitle{Dual Box Embeddings for the Description Logic {$\mathcal{EL}^{++}$}}. In \bibinfo{booktitle}{\emph{Proceedings of the {ACM} Web Conference 2024}}. \bibinfo{pages}{2250--2258}.
\newblock
\href{https://doi.org/10.1145/3589334.3645648}{doi:\nolinkurl{10.1145/3589334.3645648}}


\bibitem[Kazakov et~al\mbox{.}(2012)]%
        {kazakov2012elk}
\bibfield{author}{\bibinfo{person}{Yevgeny Kazakov}, \bibinfo{person}{Markus Kr{\"o}tzsch}, {and} \bibinfo{person}{Franti{\v{s}}ek Siman{\v{c}}{\'\i}k}.} \bibinfo{year}{2012}\natexlab{}.
\newblock \showarticletitle{{ELK} Reasoner: Architecture and Evaluation}. In \bibinfo{booktitle}{\emph{Proceedings of the 1st International Workshop on {OWL} Reasoner Evaluation (ORE 2012)}} \emph{(\bibinfo{series}{{CEUR} Workshop Proceedings}, Vol.~\bibinfo{volume}{858})}.
\newblock


\bibitem[Kazakov and Sko{\v{c}}ovsk{\'y}(2018)]%
        {PULi}
\bibfield{author}{\bibinfo{person}{Yevgeny Kazakov} {and} \bibinfo{person}{Peter Sko{\v{c}}ovsk{\'y}}.} \bibinfo{year}{2018}\natexlab{}.
\newblock \showarticletitle{Enumerating Justifications Using Resolution}. In \bibinfo{booktitle}{\emph{International Joint Conference on Automated Reasoning (IJCAR 2018)}}. \bibinfo{publisher}{Springer}, \bibinfo{pages}{609--626}.
\newblock


\bibitem[Klarman et~al\mbox{.}(2011)]%
        {klarman2011abox}
\bibfield{author}{\bibinfo{person}{Szymon Klarman}, \bibinfo{person}{Ulle Endriss}, {and} \bibinfo{person}{Stefan Schlobach}.} \bibinfo{year}{2011}\natexlab{}.
\newblock \showarticletitle{{ABox} Abduction in the Description Logic {$\mathcal{ALC}$}}.
\newblock \bibinfo{journal}{\emph{Journal of Automated Reasoning}} \bibinfo{volume}{46}, \bibinfo{number}{1} (\bibinfo{year}{2011}), \bibinfo{pages}{43--80}.
\newblock


\bibitem[Koopmann(2021)]%
        {koopmann2021signature}
\bibfield{author}{\bibinfo{person}{Patrick Koopmann}.} \bibinfo{year}{2021}\natexlab{}.
\newblock \showarticletitle{Signature-Based Abduction with Fresh Individuals and Complex Concepts for Description Logics}. In \bibinfo{booktitle}{\emph{Proceedings of the 30th International Joint Conference on Artificial Intelligence (IJCAI 2021)}}. \bibinfo{pages}{1929--1935}.
\newblock


\bibitem[Koopmann et~al\mbox{.}(2020)]%
        {koopmann2020signature}
\bibfield{author}{\bibinfo{person}{Patrick Koopmann}, \bibinfo{person}{Warren Del-Pinto}, \bibinfo{person}{Sophie Tourret}, {and} \bibinfo{person}{Renate~A. Schmidt}.} \bibinfo{year}{2020}\natexlab{}.
\newblock \showarticletitle{Signature-Based Abduction for Expressive Description Logics}. In \bibinfo{booktitle}{\emph{Proceedings of the 17th International Conference on Principles of Knowledge Representation and Reasoning (KR 2020)}}. \bibinfo{pages}{592--602}.
\newblock


\bibitem[Kulmanov et~al\mbox{.}(2019)]%
        {kulmanov_embeddings_2019}
\bibfield{author}{\bibinfo{person}{Maxat Kulmanov}, \bibinfo{person}{Wang Liu-Wei}, \bibinfo{person}{Yuan Yan}, {and} \bibinfo{person}{Robert Hoehndorf}.} \bibinfo{year}{2019}\natexlab{}.
\newblock \showarticletitle{{EL} Embeddings: Geometric Construction of Models for the Description Logic {$\mathcal{EL}^{++}$}}. In \bibinfo{booktitle}{\emph{Proceedings of the 28th International Joint Conference on Artificial Intelligence (IJCAI 2019)}}. \bibinfo{pages}{6103--6109}.
\newblock
\href{https://doi.org/10.24963/ijcai.2019/845}{doi:\nolinkurl{10.24963/ijcai.2019/845}}


\bibitem[Kulmanov et~al\mbox{.}(2021)]%
        {kulmanov2021semantic}
\bibfield{author}{\bibinfo{person}{Maxat Kulmanov}, \bibinfo{person}{Fatima~Zohra Smaili}, \bibinfo{person}{Xin Gao}, {and} \bibinfo{person}{Robert Hoehndorf}.} \bibinfo{year}{2021}\natexlab{}.
\newblock \showarticletitle{Semantic similarity and machine learning with ontologies}.
\newblock \bibinfo{journal}{\emph{Briefings in Bioinformatics}} \bibinfo{volume}{22}, \bibinfo{number}{4} (\bibinfo{year}{2021}), \bibinfo{pages}{bbaa199}.
\newblock


\bibitem[Lo et~al\mbox{.}(2024)]%
        {lo2024end}
\bibfield{author}{\bibinfo{person}{Andy Lo}, \bibinfo{person}{Albert~Q. Jiang}, \bibinfo{person}{Wenda Li}, {and} \bibinfo{person}{Mateja Jamnik}.} \bibinfo{year}{2024}\natexlab{}.
\newblock \showarticletitle{End-to-End Ontology Learning with Large Language Models}. In \bibinfo{booktitle}{\emph{Advances in Neural Information Processing Systems 37 (NeurIPS 2024)}}. \bibinfo{pages}{87184--87225}.
\newblock


\bibitem[Pe{\~n}aloza(2020)]%
        {penaloza2020axiom}
\bibfield{author}{\bibinfo{person}{Rafael Pe{\~n}aloza}.} \bibinfo{year}{2020}\natexlab{}.
\newblock \showarticletitle{Axiom Pinpointing}.
\newblock In \bibinfo{booktitle}{\emph{Applications and Practices in Ontology Design, Extraction, and Reasoning}}. \bibinfo{series}{Studies on the Semantic Web}, Vol.~\bibinfo{volume}{49}. \bibinfo{publisher}{IOS Press}, \bibinfo{pages}{162--177}.
\newblock


\bibitem[Pukancov{\'a} and Homola(2020)]%
        {pukancova2020aaa}
\bibfield{author}{\bibinfo{person}{J{\'u}lia Pukancov{\'a}} {and} \bibinfo{person}{Martin Homola}.} \bibinfo{year}{2020}\natexlab{}.
\newblock \showarticletitle{The {AAA} {ABox} Abduction Solver}.
\newblock \bibinfo{journal}{\emph{KI -- K{\"u}nstliche Intelligenz}} \bibinfo{volume}{34}, \bibinfo{number}{4} (\bibinfo{year}{2020}), \bibinfo{pages}{517--522}.
\newblock


\bibitem[Reimers and Gurevych(2019)]%
        {reimers2019sentence}
\bibfield{author}{\bibinfo{person}{Nils Reimers} {and} \bibinfo{person}{Iryna Gurevych}.} \bibinfo{year}{2019}\natexlab{}.
\newblock \showarticletitle{Sentence-BERT: Sentence Embeddings using Siamese BERT-Networks}. In \bibinfo{booktitle}{\emph{Proceedings of the 2019 Conference on Empirical Methods in Natural Language Processing}}. \bibinfo{pages}{3982--3992}.
\newblock


\bibitem[Shi et~al\mbox{.}(2024)]%
        {shi2024taxonomy}
\bibfield{author}{\bibinfo{person}{Jingchuan Shi}, \bibinfo{person}{Hang Dong}, \bibinfo{person}{Jiaoyan Chen}, \bibinfo{person}{Zhe Wu}, {and} \bibinfo{person}{Ian Horrocks}.} \bibinfo{year}{2024}\natexlab{}.
\newblock \showarticletitle{Taxonomy completion via implicit concept insertion}. In \bibinfo{booktitle}{\emph{Proceedings of the ACM Web Conference 2024}}. \bibinfo{pages}{2159--2169}.
\newblock


\bibitem[Smaili et~al\mbox{.}(2019)]%
        {DBLP:journals/bioinformatics/SmailiGH19}
\bibfield{author}{\bibinfo{person}{Fatima~Zohra Smaili}, \bibinfo{person}{Xin Gao}, {and} \bibinfo{person}{Robert Hoehndorf}.} \bibinfo{year}{2019}\natexlab{}.
\newblock \showarticletitle{OPA2Vec: combining formal and informal content of biomedical ontologies to improve similarity-based prediction}.
\newblock \bibinfo{journal}{\emph{Bioinformatics}} \bibinfo{volume}{35}, \bibinfo{number}{12} (\bibinfo{year}{2019}), \bibinfo{pages}{2133--2140}.
\newblock
\href{https://doi.org/10.1093/bioinformatics/bty933}{doi:\nolinkurl{10.1093/bioinformatics/bty933}}


\bibitem[Sun et~al\mbox{.}(2024)]%
        {DBLP:journals/pvldb/SunHSXYDTC24}
\bibfield{author}{\bibinfo{person}{Yushi Sun}, \bibinfo{person}{Hao Xin}, \bibinfo{person}{Kai Sun}, \bibinfo{person}{Yifan~Ethan Xu}, \bibinfo{person}{Xiao Yang}, \bibinfo{person}{Xin~Luna Dong}, \bibinfo{person}{Nan Tang}, {and} \bibinfo{person}{Lei Chen}.} \bibinfo{year}{2024}\natexlab{}.
\newblock \showarticletitle{Are Large Language Models a Good Replacement of Taxonomies?}
\newblock \bibinfo{journal}{\emph{Proceedings of the {VLDB} Endowment}} \bibinfo{volume}{17}, \bibinfo{number}{11} (\bibinfo{year}{2024}), \bibinfo{pages}{2919--2932}.
\newblock
\href{https://doi.org/10.14778/3681954.3681973}{doi:\nolinkurl{10.14778/3681954.3681973}}


\bibitem[Wei-Kleiner et~al\mbox{.}(2014)]%
        {wei2014abduction}
\bibfield{author}{\bibinfo{person}{Fang Wei-Kleiner}, \bibinfo{person}{Zlatan Dragisic}, {and} \bibinfo{person}{Patrick Lambrix}.} \bibinfo{year}{2014}\natexlab{}.
\newblock \showarticletitle{Abduction Framework for Repairing Incomplete {$\mathcal{EL}$} Ontologies: Complexity Results and Algorithms}. In \bibinfo{booktitle}{\emph{Proceedings of the AAAI Conference on Artificial Intelligence}}, Vol.~\bibinfo{volume}{28}. \bibinfo{pages}{1120--1127}.
\newblock


\bibitem[Xiong et~al\mbox{.}(2022)]%
        {xiong_faithiful_2022}
\bibfield{author}{\bibinfo{person}{Bo Xiong}, \bibinfo{person}{Nico Potyka}, \bibinfo{person}{Trung-Kien Tran}, \bibinfo{person}{Mojtaba Nayyeri}, {and} \bibinfo{person}{Steffen Staab}.} \bibinfo{year}{2022}\natexlab{}.
\newblock \showarticletitle{Faithful Embeddings for {$\mathcal{EL}^{++}$} Knowledge Bases}. In \bibinfo{booktitle}{\emph{The Semantic Web -- ISWC 2022}}. \bibinfo{publisher}{Springer}, \bibinfo{pages}{22--38}.
\newblock


\bibitem[Yang et~al\mbox{.}(2025b)]%
        {OnT}
\bibfield{author}{\bibinfo{person}{Hui Yang}, \bibinfo{person}{Jiaoyan Chen}, \bibinfo{person}{Yuan He}, \bibinfo{person}{Yongsheng Gao}, {and} \bibinfo{person}{Ian Horrocks}.} \bibinfo{year}{2025}\natexlab{b}.
\newblock \showarticletitle{Language Models as Ontology Encoders}. In \bibinfo{booktitle}{\emph{The Semantic Web -- ISWC 2025}}. \bibinfo{publisher}{Springer}, \bibinfo{pages}{443--461}.
\newblock


\bibitem[Yang et~al\mbox{.}(2025a)]%
        {yang2025transbox}
\bibfield{author}{\bibinfo{person}{Hui Yang}, \bibinfo{person}{Jiaoyan Chen}, {and} \bibinfo{person}{Uli Sattler}.} \bibinfo{year}{2025}\natexlab{a}.
\newblock \showarticletitle{{TransBox}: {$\mathcal{EL}^{++}$}-closed Ontology Embedding}. In \bibinfo{booktitle}{\emph{Proceedings of the {ACM} on Web Conference 2025}}. \bibinfo{pages}{22--34}.
\newblock
\href{https://doi.org/10.1145/3696410.3714672}{doi:\nolinkurl{10.1145/3696410.3714672}}


\bibitem[Yang et~al\mbox{.}(2026)]%
        {huiWWW26}
\bibfield{author}{\bibinfo{person}{Hui Yang}, \bibinfo{person}{Jiaoyan Chen}, {and} \bibinfo{person}{Uli Sattler}.} \bibinfo{year}{2026}\natexlab{}.
\newblock \showarticletitle{Large Language Model for {OWL} Proofs}. In \bibinfo{booktitle}{\emph{Proceedings of the {ACM} Web Conference 2026}}. \bibinfo{pages}{3952--3963}.
\newblock
\href{https://doi.org/10.1145/3774904.3792395}{doi:\nolinkurl{10.1145/3774904.3792395}}


\bibitem[Yang et~al\mbox{.}(2022)]%
        {yang2022hypergraph}
\bibfield{author}{\bibinfo{person}{Hui Yang}, \bibinfo{person}{Yue Ma}, {and} \bibinfo{person}{Nicole Bidoit}.} \bibinfo{year}{2022}\natexlab{}.
\newblock \showarticletitle{Hypergraph-Based Inference Rules for Computing {$\mathcal{EL}^{+}$}-Ontology Justifications}. In \bibinfo{booktitle}{\emph{International Joint Conference on Automated Reasoning (IJCAR 2022)}}. \bibinfo{publisher}{Springer}, \bibinfo{pages}{310--328}.
\newblock


\bibitem[Zhang et~al\mbox{.}(2026)]%
        {zhang2025ontourl}
\bibfield{author}{\bibinfo{person}{Xiao Zhang}, \bibinfo{person}{Huiyuan Lai}, \bibinfo{person}{Qianru Meng}, {and} \bibinfo{person}{Johan Bos}.} \bibinfo{year}{2026}\natexlab{}.
\newblock \showarticletitle{{OntoURL}: A benchmark for evaluating large language models on symbolic ontological understanding, reasoning and learning}.
\newblock \bibinfo{journal}{\emph{Journal of Web Semantics}}  \bibinfo{volume}{91} (\bibinfo{year}{2026}), \bibinfo{pages}{100886}.
\newblock


\bibitem[Zhao et~al\mbox{.}(2026)]%
        {zhao2026subsumption}
\bibfield{author}{\bibinfo{person}{Haoruo Zhao}, \bibinfo{person}{Wenshuo Tang}, \bibinfo{person}{Duncan Guthrie}, \bibinfo{person}{Michele Sevegnani}, \bibinfo{person}{David Flynn}, {and} \bibinfo{person}{Paul Harvey}.} \bibinfo{year}{2026}\natexlab{}.
\newblock \showarticletitle{From Subsumption to Satisfiability: LLM-Assisted Active Learning for OWL Ontologies}.
\newblock \bibinfo{journal}{\emph{arXiv preprint arXiv:2604.16672}} (\bibinfo{year}{2026}).
\newblock


\end{thebibliography}

\appendix
\AtBeginEnvironment{table}{\renewcommand{\arraystretch}{0.88}}
\setlength{\aboverulesep}{0.2ex}\setlength{\belowrulesep}{0.4ex}

\section{Prompts}\label{app:prompts}
\suppressfloats[t]\paragraph{\ours} The prompts of \ours and of LLM-direct are listed below. \texttt{\{A\}} and \texttt{\{B\}} are replaced with concept labels ($\exists r.X$ as ``\texttt{r some X}''), \texttt{\{ontology\}} with the ontology name{,} and \texttt{\{n\}} and \texttt{\{k\}} with the number of candidates.

\begin{promptbox}{Stage 2a (atomic datasets)}
\scriptsize
\begin{verbatim}
You are an ontology expert. Given the following concepts from
a biomedical/food ontology, determine subsumption
relationships.

Query: Is "{A}" a subclass of "{B}"?

We know that the following {n} concepts are direct subclasses
of "{B}" in the ontology, ranked by embedding similarity:
- {candidate_1}
- {candidate_2}
...

For each candidate below, determine independently if "{A}" is
a subclass of that candidate. Answer YES or NO for each.

1. Is "{A}" a subclass of "{candidate_1}"?
2. Is "{A}" a subclass of "{candidate_2}"?
...

Respond in the format:
1. YES/NO
2. YES/NO
...
Only output the numbered YES/NO answers, nothing else.
\end{verbatim}
\end{promptbox}
Stages~2b and~3a use analogous prompts with the direct superclasses of $A$ (Stage~2b, asking whether each candidate is a subclass of $B$) or the top-ranked atomic concepts (Stage~3a, asking both directions) as candidates.

\begin{promptbox}{Stages 2a and 3a (existential datasets)}
\scriptsize
\begin{verbatim}
You are an ontology expert analyzing {ontology} concept
relationships.

Query: Is "{A}" a subclass of "{B}"?

Below are {k} candidate intermediate concepts C, each of the
form "role some Filler". The ontology may be incomplete, so
use your own domain knowledge as well.

For each candidate C, answer YES if both hold:
1. "{A}" is a subclass of C;
2. C is a subclass of "{B}", and C is not equivalent to "{A}"
   or to "{B}".

{numbered list of the k candidates, e.g. "1. r some X"}

Respond in the format:
1. YES/NO
2. YES/NO
...
Only output the numbered YES/NO answers, nothing else.
\end{verbatim}
\end{promptbox}
Candidates: $\exists r.X$ for the direct children $X$ of $B_1$ (Stage~2a) or the top-ranked restrictions (Stage~3a).

\begin{promptbox}{Stage 3b and LLM-direct (decision)}
\scriptsize
\begin{verbatim}
You are an ontology expert in the biomedical and food domains.
Determine whether the following subsumption relationship
holds.

Is "{A}" a subclass of "{B}"?

Answer with only YES or NO.
\end{verbatim}
\end{promptbox}

\paragraph{LLM-direct} The following prompt is used for the LLM-direct baseline when the LLM responds \texttt{YES} to the Stage 3b prompt above, with \texttt{YES} already appended after the final ``Answer:''. Lists exceeding 50 items are randomly sampled to 50 using a fixed seed. Sections highlighted in blue apply only to the atomic datasets, while those highlighted in amber apply only to the existential datasets.

\begin{promptbox}{LLM-direct, bridge concepts}
\scriptsize
\begin{verbatim}
You are an ontology expert analyzing {ontology} concept
relationships.

Query: Is "{A}" a subclass of "{B}"?

The ontology is incomplete: some axioms are missing, so the
subsumption may hold even if it cannot be derived from the
information below. Use your own domain knowledge as well. The
following information comes from the ontology.

Direct superclasses of "{A}":
{direct superclasses of A, one per line}

Direct subclasses of "{B/B1}":
{direct subclasses of B/B1, one per line}
\end{verbatim}
\begin{pvariant}{pvExist}{existential datasets}
\begin{verbatim}
Relations (object properties) of the ontology:
{the roles of O' (outside fillers of restrictions), one per line}
\end{verbatim}
\end{pvariant}
\begin{verbatim}
Task:
1. Decide whether "{A}" is a subclass of "{B}".
\end{verbatim}
\begin{pvariant}{pvAtom}{atomic datasets}
\begin{verbatim}
2. If YES, choose the intermediate concepts C from the
   concepts listed above, i.e. the direct superclasses of
   "{A}" and the direct subclasses of "{B}" such that "{A}" is
   a subclass of C and C is a subclass of "{B}". For a direct
   superclass of "{A}", this means that it is a subclass of
   "{B}"; for a direct subclass of "{B}", this means that
   "{A}" is a subclass of it. List every listed concept that
   is such a C, the most likely first, with the names written
   exactly as listed and without quotes. Write NONE only if
   none of the listed concepts is such a C.
\end{verbatim}
\end{pvariant}
\begin{pvariant}{pvExist}{existential datasets}
\begin{verbatim}
2. If YES, choose the intermediate concepts C of the form
   "<role> some <Filler>" whose Filler is one of the concepts
   listed above (a direct superclass of "{A}" or a direct
   subclass of "{B1}") and whose role is one of the
   relations listed above, such that "{A}" is a subclass of C
   and C is a subclass of "{B}", and C is not equivalent to
   "{A}" or to "{B}". List every such C, the most likely
   first, with the names written exactly as listed and without
   quotes. Write NONE only if there is no such C.
\end{verbatim}
\end{pvariant}
\begin{verbatim}
Respond in exactly this format:
Answer: YES or NO
Intermediate concepts:
1. <C>
2. <C>
...
If the answer is NO, write only "Answer: NO".

Answer:
\end{verbatim}
\end{promptbox}

\section{Datasets}\label{app:data}

Table~\ref{tab:ontologies} reports the statistics of the generated datasets, each with 1,000 positive{, 1,000 random} and 1,000 hard negative samples. $|\mathcal{P}_{\text{all}}|$ is the number of all possible positive samples (on Snomed$_\exists$ and GO$^{+}_\exists$, of the positive queries before sampling). The column Splits gives the numbers of positive samples that can be solved in Stages~2a, 2b and 3a (they may overlap) and of those without any bridge concept (3b); on the existential datasets, a positive sample can be solved in Stage~2a iff $\exists r.X$ is a ground-truth bridge concept for a direct child $X$ of $B_1$ in $\mathcal{O}'$, and always in Stage~3a.

\begin{table}[t]
\centering
\caption{Dataset statistics.}
\label{tab:ontologies}
\small
\begin{tabular}{cccc}
\toprule
\textbf{Dataset} & \textbf{Concepts} & \textbf{$|\mathcal{P}_{\text{all}}|$} & \textbf{Splits} (2a/2b/3a/3b)\\
\midrule
FoodOn$_A$         & 30,995  & 102{,}764     & 814/778/978/22 \\
Snomed$_A$         & 364,352 & 1{,}184{,}461 & 654/569/900/100 \\
GO$^{+}_{A}$ & 84,016 & 188{,}523 & 790/781/976/24 \\
\midrule
Snomed$_\exists$   & 364,352 & 2{,}410       & 278/--/1{,}000/0 \\
GO$^{+}_\exists$ & 84,016 & 22{,}022 & 263/--/1{,}000/0 \\
\bottomrule
\end{tabular}
\end{table}

In the dataset construction described in Section~\ref{sec:data_ex}, we set the deletion rate $\rho$ to 30\% for Snomed CT and 10\% for GO$^{+}$, which are sufficient to generate 1,000 positive samples for each dataset.

For simplicity, we restrict the existential dataset experiments to roles occurring in existential restrictions that are not nested within other existential restrictions. Specifically, Snomed CT contains 20 such roles, while GO$^{+}$ contains 194. Stage~3a uses the corresponding non-nested roles in $\mathcal{O}'$: 20 for Snomed CT and 190 for GO$^{+}$ (192 before merging two pairs of roles with identical labels).
Furthermore, we exclude reflexive roles, such as \emph{Is modification of} in Snomed CT, from the queries, since reflexivity guarantees that $A \sqsubseteq \exists r.A$ holds for any concept $A$, thereby yielding trivial bridge concepts.

\paragraph{Semantic Audit of Negative Samples}
As discussed in Section~\ref{sec:data}, non-entailed subsumptions are not necessarily false. We therefore conduct a semantic audit using an independent LLM, Claude Opus 5.5, which is not used by any compared method. Each sample is presented using only concept labels, without its ground-truth label, and classified as \emph{valid}, \emph{invalid}, or \emph{uncertain}. We first audit 50 random and 50 hard negatives per dataset, together with 10 positive controls{ per dataset}. Following our predefined criterion, full filtering is performed when more than 5\% of sampled negatives are flagged as valid or uncertain.

As shown in Table~\ref{tab:audit}, all sampled random negatives are judged invalid, whereas hard negatives contain potentially valid subsumptions across all datasets. We therefore audit all hard negatives, removing samples flagged as valid or uncertain (9.8\% overall). Table~\ref{tab:all_results} reports the corresponding results under Hard Neg (filtered). Filtering slightly improves performance without changing the ranking of methods, supporting the robustness of our findings. Nevertheless, LLM-based auditing does not guarantee correctness, and the original results remain our primary evaluation.

\begin{table}[t]
\centering
\caption{Semantic audit of the negative samples by Claude Opus 5.5 (V/U: numbers judged valid/uncertain; Pos.: positive controls judged valid).}
\label{tab:audit}
\small
\begin{tabular}{@{}ccccccc@{}}
\toprule
& \multicolumn{3}{c}{\textbf{Sample (50 per type)}} & \multicolumn{3}{c}{\textbf{Full audit (hard)}} \\
\cmidrule(lr){2-4} \cmidrule(lr){5-7}
\textbf{Dataset} & Random V/U & Hard V/U & Pos. & V & U & Removed \\
\midrule
FoodOn$_A$ & 0/0 & 0/3 & 10/10 & 27 & 32 & 59 (5.9\%) \\
Snomed$_A$ & 0/0 & 1/2 & 8/10 & 31 & 50 & 81 (8.1\%) \\
GO$^{+}_{A}$ & 0/0 & 3/1 & 8/10 & 45 & 39 & 84 (8.4\%) \\
Snomed$_\exists$ & 0/0 & 3/3 & 10/10 & 55 & 47 & 102 (10.2\%) \\
GO$^{+}_\exists$ & 0/0 & 2/7 & 9/10 & 44 & 118 & 162 (16.2\%) \\
\midrule
Total & 0/0 & 9/16 & 45/50 & 202 & 286 & 488 (9.8\%) \\
\bottomrule
\end{tabular}
\end{table}

\section{Extra Results}

\subsection{Direct Children/Parents vs.\ All}\label{app:direct_all}

In this subsection, we compare the performance impact of using only direct children/parents as candidates in Stage 2a/2b (denoted as \textit{DIRECT}, the default setting of \ours) versus using all transitive parent/child nodes (denoted as \textit{ALL}).

We first present statistics for the number of direct and transitive parent/child nodes associated with the evaluated positive samples. As shown in Figure~\ref{fig:distri_direct}, the average number of direct parents is approximately 1.0--1.4, which is substantially lower than the average number of direct children (8.6--32.7). This observation aligns with the structural properties of ontologies: a concept typically has only a few direct parents but may have many direct children.

\begin{figure}[tb]
    \centering
    \includegraphics[width=\linewidth]{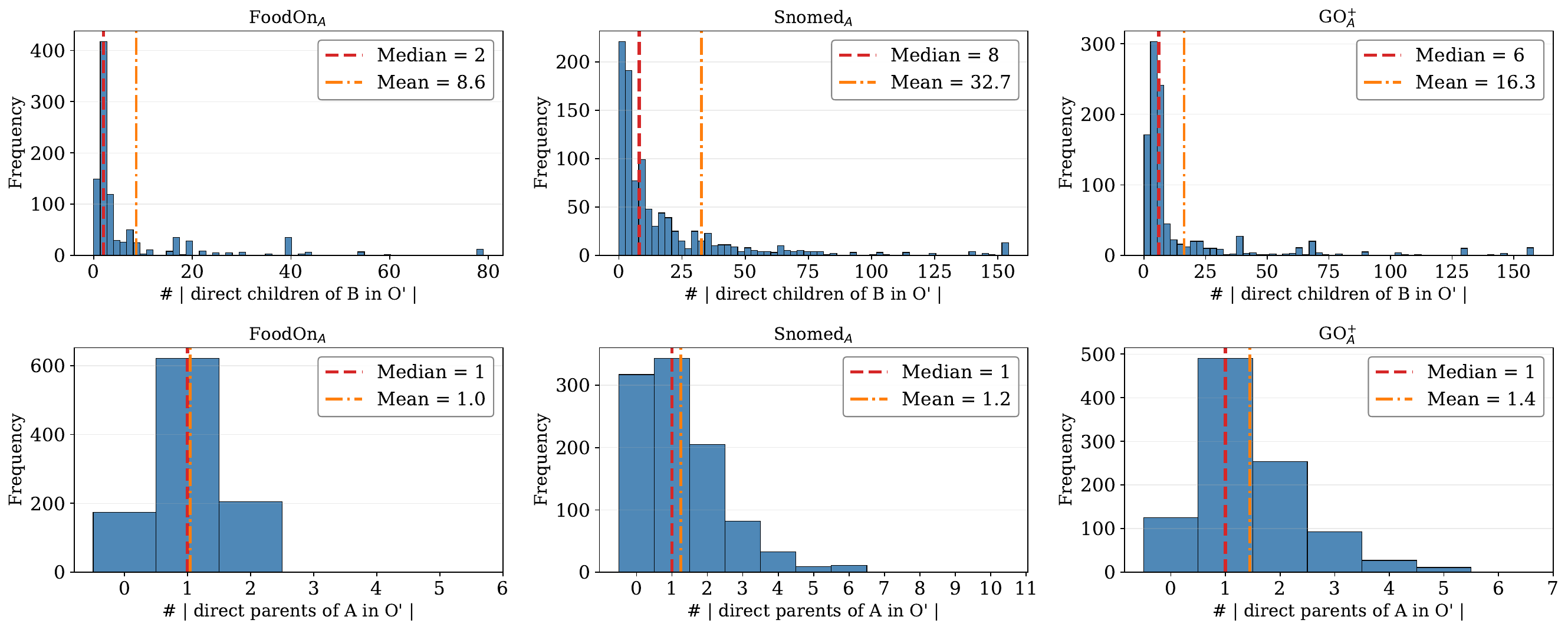}
    \caption{Distribution of the number of direct children and parents for the positive samples in the test sets of FoodOn$_A$, Snomed$_A$ and GO$^{+}_{A}$.}
    \label{fig:distri_direct}
\end{figure}

Furthermore, Figure~\ref{fig:distri_direct_vs_all} illustrates that the number of transitive (\textit{ALL}) parents/children can be significantly larger than the number of direct ones, in some cases by more than three orders of magnitude. This highlights the substantial expansion of the candidate space when transitive relationships are considered.

\begin{figure}[tb]
    \centering
    \includegraphics[width=\linewidth]{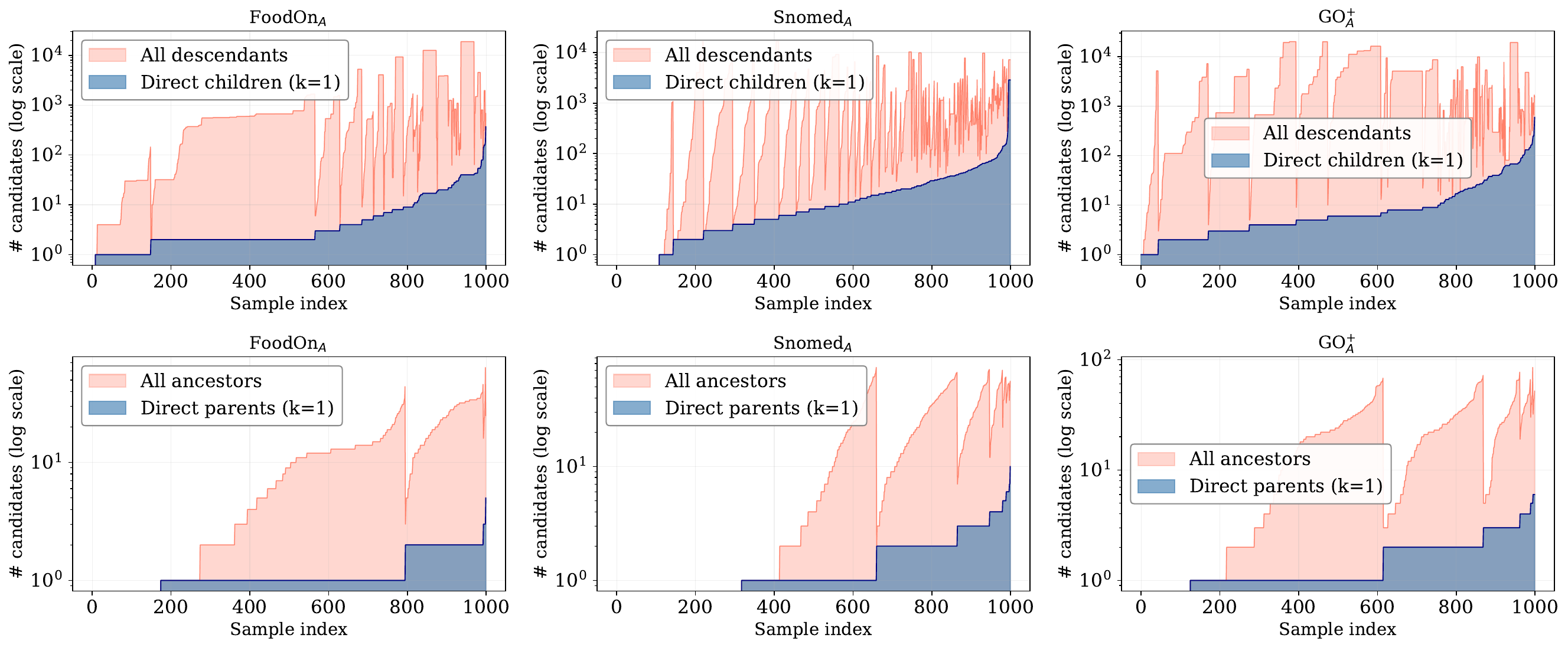}
    \caption{Comparison between the numbers of direct and all children and parents on FoodOn$_A$, Snomed$_A$ and GO$^{+}_{A}$. Samples are ordered by increasing number of direct parents/children.}
    \label{fig:distri_direct_vs_all}
\end{figure}

In Table~\ref{tab:hitk_per_group}, we present the retrieval performance under the \textit{DIRECT} 
and \textit{ALL} settings. As expected, the \textit{DIRECT} setting consistently outperforms or matches the \textit{ALL} setting across all cases, demonstrating the advantage of restricting candidates to direct parents/children.

\begin{table}[t]
\centering
\caption{Retrieval performance with direct (\textit{DIRECT}) or all (\textit{ALL}) parents/children{; the better value of each pair in bold}.
}
\label{tab:hitk_per_group}
\small
{\setlength{\tabcolsep}{4pt}
\begin{tabular}{@{}ccc ccccc@{}}
\toprule
\textbf{\ours} & \textbf{Stage} & & \textbf{H@1} & \textbf{H@5} & \textbf{H@10} & \textbf{H@30} & \textbf{H@100} \\
\midrule
\rowcolor{tabData}\multicolumn{8}{c}{\textit{FoodOn$_A$}} \\
\multirow{4}{*}{(OnT$^*$)} & \multirow{2}{*}{2a} & \textit{DIRECT} & \textbf{0.547} & \textbf{0.964} & \textbf{0.996} & \textbf{1.000} & \textbf{1.000} \\
 &  & \textit{ALL} & 0.172 & 0.393 & 0.515 & 0.686 & 0.829 \\
 & \multirow{2}{*}{2b} & \textit{DIRECT} & \textbf{0.920} & \textbf{1.000} & \textbf{1.000} & 1.000 & 1.000 \\
 &  & \textit{ALL} & 0.878 & 0.963 & 0.982 & 1.000 & 1.000 \\
\cmidrule(l){1-8}
\multirow{4}{*}{(SBERT)} & \multirow{2}{*}{2a} & \textit{DIRECT} & \textbf{0.484} & \textbf{0.892} & \textbf{0.940} & \textbf{0.998} & \textbf{1.000} \\
 &  & \textit{ALL} & 0.082 & 0.178 & 0.248 & 0.402 & 0.609 \\
 & \multirow{2}{*}{2b} & \textit{DIRECT} & \textbf{0.932} & \textbf{1.000} & \textbf{1.000} & \textbf{1.000} & 1.000 \\
 &  & \textit{ALL} & 0.824 & 0.922 & 0.974 & 0.999 & 1.000 \\
\midrule
\rowcolor{tabData}\multicolumn{8}{c}{\textit{Snomed$_A$}} \\
\multirow{4}{*}{(OnT$^*$)} & \multirow{2}{*}{2a} & \textit{DIRECT} & \textbf{0.664} & \textbf{0.950} & \textbf{0.974} & \textbf{0.995} & \textbf{1.000} \\
 &  & \textit{ALL} & 0.304 & 0.673 & 0.798 & 0.920 & 0.983 \\
 & \multirow{2}{*}{2b} & \textit{DIRECT} & \textbf{0.886} & \textbf{1.000} & \textbf{1.000} & \textbf{1.000} & 1.000 \\
 &  & \textit{ALL} & 0.545 & 0.775 & 0.870 & 0.963 & 1.000 \\
\cmidrule(l){1-8}
\multirow{4}{*}{(SBERT)} & \multirow{2}{*}{2a} & \textit{DIRECT} & \textbf{0.531} & \textbf{0.887} & \textbf{0.951} & \textbf{0.982} & \textbf{0.997} \\
 &  & \textit{ALL} & 0.245 & 0.514 & 0.618 & 0.739 & 0.872 \\
 & \multirow{2}{*}{2b} & \textit{DIRECT} & \textbf{0.903} & \textbf{1.000} & \textbf{1.000} & \textbf{1.000} & 1.000 \\
 &  & \textit{ALL} & 0.554 & 0.831 & 0.912 & 0.979 & 1.000 \\
\bottomrule
\end{tabular}
}
\end{table}

In Table~\ref{tab:direct_vs_transitive_f1}, we report the overall performance of the \textit{DIRECT} and \textit{ALL} settings (the \textit{DIRECT} values are those of Table~\ref{tab:all_results}). While the \textit{ALL} setting achieves higher F1 scores in five of the eight cases (all four with hard negatives), the \textit{DIRECT} setting demonstrates superior overall performance. This advantage is particularly evident in the consistently higher X-F1 scores across all cases.

\begin{table}[t]
\centering
\caption{Overall results of \textit{DIRECT} vs.\
  \textit{ALL}; the better value of each pair in bold.
}
\label{tab:direct_vs_transitive_f1}
\small
{\setlength{\tabcolsep}{4pt}
\begin{tabular}{@{}cc ccc ccc@{}}
\toprule
& & \multicolumn{3}{c}{\textbf{Random Neg}} & \multicolumn{3}{c}{\textbf{Hard Neg}} \\
\cmidrule(lr){3-5} \cmidrule(lr){6-8}
\textbf{\ours} & & \textbf{F1} & \textbf{X-F1} & \textbf{X-F1$^*$} & \textbf{F1} & \textbf{X-F1} & \textbf{X-F1$^*$} \\
\midrule
\rowcolor{tabData}\multicolumn{8}{c}{\textit{FoodOn$_A$}} \\
\multirow{2}{*}{(OnT$^*$)} & \textit{DIRECT} & 0.960 & \textbf{0.894} & \textbf{0.793} & 0.845 & \textbf{0.781} & \textbf{0.683} \\
 & \textit{ALL} & \textbf{0.968} & 0.876 & 0.750 & \textbf{0.868} & 0.779 & 0.658 \\
\cmidrule(l){1-8}
\multirow{2}{*}{(SBERT)} & \textit{DIRECT} & \textbf{0.960} & \textbf{0.885} & \textbf{0.788} & 0.844 & \textbf{0.772} & \textbf{0.679} \\
 & \textit{ALL} & 0.953 & 0.838 & 0.757 & \textbf{0.855} & 0.743 & 0.667 \\
\midrule
\rowcolor{tabData}\multicolumn{8}{c}{\textit{Snomed$_A$}} \\
\multirow{2}{*}{(OnT$^*$)} & \textit{DIRECT} & \textbf{0.968} & \textbf{0.849} & \textbf{0.662} & 0.862 & \textbf{0.747} & 0.571 \\
 & \textit{ALL} & 0.957 & 0.778 & 0.644 & \textbf{0.871} & 0.697 & 0.571 \\
\cmidrule(l){1-8}
\multirow{2}{*}{(SBERT)} & \textit{DIRECT} & \textbf{0.970} & \textbf{0.848} & 0.649 & 0.856 & \textbf{0.738} & 0.553 \\
 & \textit{ALL} & 0.962 & 0.806 & \textbf{0.672} & \textbf{0.870} & 0.718 & \textbf{0.593} \\
\bottomrule
\end{tabular}
}
\end{table}

However, for the X-F1$^*$ metric, the \textit{ALL} setting outperforms \textit{DIRECT} in some scenarios, such as the hard-negative setting of $\ours_{\text{SBERT}}$ on Snomed$_A$. This may be because the candidates in the \textit{DIRECT} setting are highly relevant and difficult to distinguish, which can lead the LLM to mistakenly verify them as true bridge concepts. As a result, the X-F1$^*$ score decreases, since this metric requires all identified bridge concepts to be correct{.}

\subsection{Ablation of Embedding Models and Retrieval Size}\label{app:ablation}
\suppressfloats[t]

\begin{sloppypar}To assess the contributions of different components of \ours, we conduct a systematic ablation study along two dimensions: (1) the base embedding model used for OnT, where we evaluate two alternatives to the default \texttt{all-MiniLM-L12-v2} (384-dimensional): \texttt{all-MiniLM-L6-v2} (384-dimensional) and \texttt{all-mpnet-base-v2} (768-dimensional); and (2) the candidate retrieval size $k$, where we consider $k \in \{5, 20\}$ in addition to the default setting of $k=10$.\end{sloppypar}

The results are summarized in Table~\ref{tab:ablation}. The default configurations correspond to $\ours_{\text{OnT}}^*$ in Table~\ref{tab:all_results}, while the remaining configurations are evaluated in separate runs. Overall, the larger embedding model, \texttt{all-mpnet-base-v2}, achieves slightly better performance in the hard-negative setting, although the improvements are marginal. Similarly, varying the candidate retrieval size $k$ has only a limited effect on performance. These findings suggest that \ours is relatively robust to the choice of embedding model and candidate retrieval size.

\begin{table}[t]
\centering
\caption{Ablation study results on FoodOn$_A$ (default setting shaded).}
\label{tab:ablation}
\small
\setlength{\tabcolsep}{4pt}
\begin{tabular}{@{}cc ccc ccc@{}}
\toprule
\multicolumn{2}{c}{\textbf{Setting}} & \multicolumn{3}{c}{\textbf{Random Neg}} & \multicolumn{3}{c}{\textbf{Hard Neg}} \\
\cmidrule(lr){1-2} \cmidrule(lr){3-5} \cmidrule(lr){6-8}
Embed & $k$ & F1 & X-F1 & X-F1$^*$ & F1 & X-F1 & X-F1$^*$ \\
\midrule
MiniLM-L6 & 10 & 0.959 & 0.888 & 0.784 & 0.845 & 0.777 & 0.678 \\
MPNet & 10 & 0.958 & 0.888 & 0.789 & \textbf{0.858} & \textbf{0.789} & \textbf{0.695} \\
\hl{MiniLM-L12} & \hl{10} & \hl{\textbf{0.960}} & \hl{\textbf{0.894}} & \hl{\textbf{0.793}} & \hl{0.845} & \hl{0.781} & \hl{0.683} \\
\midrule
MiniLM-L12 & 5 & \textbf{0.961} & 0.887 & \textbf{0.804} & \textbf{0.852} & 0.780 & \textbf{0.701} \\
MiniLM-L12 & 20 & 0.957 & 0.893 & 0.788 & 0.842 & 0.780 & 0.680 \\
\hl{MiniLM-L12} & \hl{10} & \hl{0.960} & \hl{\textbf{0.894}} & \hl{0.793} & \hl{0.845} & \hl{\textbf{0.781}} & \hl{0.683} \\
\bottomrule
\end{tabular}
\end{table}

\begin{table}[t]
\centering
\caption{Retrieval results on FoodOn$_A$ and Snomed$_\exists$ (as in Table~\ref{tab:part1_detailed}).}
\label{tab:part1_go}
\small
\begin{tabular}{@{}ccccccc@{}}
\toprule
\textbf{Variant} & \textbf{Stage} & \textbf{H@1} & \textbf{H@5} & \textbf{H@10} & \textbf{H@30} & \textbf{H@100} \\
\midrule
\rowcolor{tabData}\multicolumn{7}{c}{\textit{FoodOn$_A$}} \\
\multirow{3}{*}{$\ours_{\text{OnT}}^*$} & 2a & 0.547 & 0.964 & 0.996 & 1.000 & 1.000 \\
 & 2b & 0.920 & 1.000 & 1.000 & 1.000 & 1.000 \\
 & 3a & 0.429 & 0.714 & 0.714 & 0.714 & 0.857 \\
\cmidrule(l){1-7}
\multirow{3}{*}{$\ours_{\text{SBERT}}$} & 2a & 0.484 & 0.892 & 0.940 & 0.998 & 1.000 \\
 & 2b & 0.932 & 1.000 & 1.000 & 1.000 & 1.000 \\
 & 3a & 0.286 & 0.857 & 0.857 & 0.857 & 0.857 \\
\midrule
\rowcolor{tabData}\multicolumn{7}{c}{\textit{Snomed$_\exists$}} \\
\multirow{2}{*}{$\ours_{\text{OnT}}^*$} & 2a & 0.486 & 0.899 & 1.000 & 1.000 & 1.000 \\
 & 3a & 0.062 & 0.145 & 0.223 & 0.377 & 0.583 \\
\cmidrule(l){1-7}
\multirow{2}{*}{$\ours_{\text{SBERT}}$} & 2a & 0.255 & 0.878 & 0.960 & 1.000 & 1.000 \\
 & 3a & 0.003 & 0.012 & 0.030 & 0.098 & 0.209 \\
\bottomrule
\end{tabular}
\end{table}

\begin{table}[t]
\centering
\caption{Stage contribution ablation on FoodOn$_A$ and Snomed$_\exists$ (as in Table~\ref{tab:stage_ablation}).}
\label{tab:stage_ablation_go}
\small
\begin{tabular}{@{}cccc ccc ccc@{}}
\toprule
\multicolumn{4}{c}{\textbf{Stages}} & \multicolumn{3}{c}{\textbf{Random Neg}} & \multicolumn{3}{c}{\textbf{Hard Neg}} \\
\cmidrule(lr){1-4} \cmidrule(lr){5-7} \cmidrule(lr){8-10}
2a & 2b & 3a & 3b & F1 & X-F1 & X-F1$^*$ & F1 & X-F1 & X-F1$^*$ \\
\midrule
\rowcolor{tabData}\multicolumn{10}{c}{\textit{FoodOn$_A$}} \\
$\checkmark$ & -- & -- & -- & 0.824 & 0.767 & 0.690 & 0.789 & 0.733 & 0.656 \\
-- & $\checkmark$ & -- & -- & 0.832 & 0.812 & 0.701 & 0.772 & 0.751 & 0.645 \\
-- & -- & $\checkmark$ & -- & 0.780 & 0.651 & 0.367 & 0.724 & 0.601 & 0.334 \\
-- & -- & -- & $\checkmark$ & 0.835 & \NA{} & \NA{} & 0.788 & \NA{} & \NA{} \\
\cmidrule(l){1-10}
$\checkmark$ & $\checkmark$ & $\checkmark$ & $\checkmark$ & \textbf{0.960} & \textbf{0.894} & \textbf{0.793} & \textbf{0.845} & \textbf{0.781} & \textbf{0.683} \\
\midrule
\rowcolor{tabData}\multicolumn{10}{c}{\textit{Snomed$_\exists$}} \\
$\checkmark$ & -- & -- & -- & 0.593 & 0.347 & 0.153 & 0.567 & 0.329 & 0.144 \\
-- & -- & $\checkmark$ & -- & 0.821 & \textbf{0.538} & 0.064 & 0.733 & \textbf{0.469} & 0.053 \\
-- & -- & -- & $\checkmark$ & 0.788 & \NA{} & \NA{} & 0.765 & \NA{} & \NA{} \\
\cmidrule(l){1-10}
$\checkmark$ & -- & $\checkmark$ & $\checkmark$ & \textbf{0.867} & 0.528 & \textbf{0.184} & \textbf{0.778} & 0.460 & \textbf{0.156} \\
\bottomrule
\end{tabular}
\end{table}

\subsection{\texorpdfstring{Results on FoodOn$_A$ and Snomed$_\exists$}{Results on FoodOn-A and Snomed-exists}}\label{app:go_results}

Tables~\ref{tab:part1_go} and~\ref{tab:stage_ablation_go} extend Tables~\ref{tab:part1_detailed} and~\ref{tab:stage_ablation}, respectively, with results for FoodOn$_A$ and Snomed$_\exists$.

In Table~\ref{tab:stage_ablation_go}, on both datasets, the full pipeline performs best in almost all metrics; the only exception is a marginally higher X-F1 of Stage~3a alone on Snomed$_\exists$ (by at most 0.010).

\clearpage\end{document}